\documentclass{article}

\usepackage{microtype}
\usepackage{graphicx}
\usepackage{subfigure}
\usepackage{booktabs}
\usepackage{subfigure}
\usepackage{subcaption}

\usepackage{hyperref}

\usepackage[accepted]{icml2024}

\usepackage{amsmath}
\usepackage{amssymb}
\usepackage{mathtools}
\usepackage{amsthm}
\usepackage{booktabs}

\usepackage[capitalize,nameinlink]{cleveref}

\theoremstyle{plain}
\newtheorem{theorem}{Theorem}[section]

\newtheorem{lemma}[theorem]{Lemma}

\theoremstyle{definition}

\newtheorem{assumption}[theorem]{Assumption}
\theoremstyle{remark}
\newtheorem{remark}[theorem]{Remark}

\usepackage[textsize=tiny]{todonotes}

\icmltitlerunning{Scaling Laws for the Value of Individual Data Points in Machine Learning}
\usepackage{math_command}

\begin{document}

\twocolumn[
\icmltitle{Scaling Laws for the Value of Individual Data Points in Machine Learning}



\icmlsetsymbol{equal}{*}

\begin{icmlauthorlist}
\icmlauthor{Ian Covert}{xxx,equal}
\icmlauthor{Wenlong Ji}{xxx,equal}
\icmlauthor{Tatsunori Hashimoto}{xxx}
\icmlauthor{James Zou}{xxx}
\end{icmlauthorlist}

\icmlaffiliation{xxx}{Stanford University}


\icmlcorrespondingauthor{Ian Covert}{icovert@stanford.edu}

\icmlkeywords{Machine Learning, ICML}

\vskip 0.3in
]



\printAffiliationsAndNotice{}  

\begin{abstract}
    Recent works have shown that machine learning models improve at a predictable rate
    with the total amount of training data, leading to \textit{scaling laws} that describe the relationship between error and dataset size.
    These scaling laws can help design
    a model's training dataset, but they typically take an aggregate view of the data by only considering the dataset's size. We
    introduce a new perspective by investigating scaling behavior for the value of \textit{individual data points}:
    we find that a data point's
    contribution to model's performance shrinks predictably with the size of the
    dataset in a log-linear manner. Interestingly, there is significant variability in the scaling exponent among different data points, indicating that certain points are more valuable in small datasets while others
    are relatively more useful as a part of large datasets. We provide learning theory to support our scaling law, and we observe empirically that it holds across diverse model classes.
    We further propose a maximum likelihood estimator and an amortized estimator to efficiently learn the individualized scaling behaviors from a small number of noisy observations per data point.
    Using our
    estimators, we provide insights into factors that influence the scaling behavior of different data points. Finally, we demonstrate applications of the individualized scaling laws to data valuation and data subset selection.
    Overall, our work represents a first step towards understanding and utilizing scaling properties for the value of individual data points.
\end{abstract}

\section{Introduction}
\label{sec:intro}
Machine learning models for vision and language have improved dramatically in recent years \citep{radford2021learning, touvron2023llama, openai2023gpt, team2023gemini}, in part due to increasing model sizes, but also due to using larger amounts of high-quality training data \citep{gadre2023datacomp}. Recent research has found that increasing the amount of training data improves models at a predictable rate, leading to \textit{scaling laws} that describe the relationship between error and dataset size \citep{kaplan2020scaling}.
For example, \citet{hoffmann2022training} introduced the following scaling law to determine Chinchilla's compute-optimal training budget,
\begin{equation}
    \mathrm{error} \approx \epsilon + \frac{a}{p^\gamma} + \frac{b}{n^\lambda}, \label{eq:aggregate-scaling}
\end{equation}
where $\epsilon$ represents irreducible error, $p$ is the number of model parameters, $n$ is the number of training examples, and $(a, b, \nu, \lambda)$ are learned constants.
These scaling laws help inform how to trade off model and data size, but they take an \textit{aggregate} view of the dataset by not differentiating between training examples; this is limiting because certain data points are more useful than others, particularly for noisy web-scraped datasets.

In the current era of datasets aggregated from heterogeneous sources, it is important to understand how individual data points or data sources affect the behavior of model training. This can give practical guidance on what type of data to prioritize, especially as the dataset grows. It also provides more fundamental insights into how the impact of different data scales with data size---e.g., which points are useful in small datasets but whose value diminishes quickly as the data size increases, and which points are relatively more valuable in large datasets. 

Our work aims to take initial steps in this direction: motivated by these questions, we propose \textit{individualized data scaling laws}, which describe how the impact of data scales with the size of the dataset for each training example.
Our analysis focuses on a data point's marginal contribution to a trained model's performance, which previous work has noted shrinks with the size of the preceding dataset \citep{kwon2021beta}; we build on this by showing that the marginal contribution
shrinks at a reliable rate, and that this rate varies between data points. The scaling behavior is represented by a simple parametric form, which is inspired by existing aggregate scaling laws (see \Cref{sec:scaling}). We find support for this phenomenon in certain classical learning settings (e.g., linear regression), and we show that it holds empirically across several model types, including logistic regression, multi-layer perceptrons (MLPs) and support vector machines (SVMs).

Analyzing individual scaling laws can ultimately help improve the dataset, for example by identifying data points
that consistently degrade a model's performance, or by surfacing data points whose contributions remain large even as the dataset size grows.
Doing so requires fitting scaling parameters for every example in a dataset, which can be computationally costly;
for example, it would be intractable to precisely estimate each point's expected marginal contribution at a range of dataset sizes, even though this is a natural way to
observe the trend from the scaling law (see \Cref{fig:validation}). We therefore consider how to make the fitting approach efficient in \Cref{sec:estimation}, where we develop procedures that let us estimate the scaling parameters for entire datasets. These include (1)~a statistical approach to estimate the scaling law given a moderate number of sampled marginal contributions, and (2)~a neural estimator that amortizes the fitting process across all examples in a dataset.

We conduct experiments to test our approach on several datasets (including tabular, text and image data) and multiple model types. We verify that the scaling law accurately describes the mean contribution at each dataset size: for example, we find that our scaling law explains ${>}0.9$ of the variance in marginal contributions between data points (see \Cref{tab:global-r2}). We also test whether it can extrapolate beyond the range where it is fit, which would enable the scaling law to be learned
in a less costly dataset range before being used in a larger regime. Finally, we explore applications of the scaling laws to data subset selection, where we find that they identify useful
new data points that depend on the current dataset size; and we demonstrate their application to data valuation, which we show is closely related to our scaling analysis \citep{ghorbani2019data}.

\textbf{Our contributions.} The main contributions of this work are
the following:
(1)~we propose and find evidence for individualized data scaling laws, (2)~we show how to fit the scaling behavior
using a small number of noisy observations per data point, (3)~we provide qualitative understanding of factors that influence scaling behavior, and (4)~we demonstrate that individualized scaling laws can facilitate data valuation and data subset selection. Overall, the scaling behavior of individual data points provides a new tool for understanding and improving training data for machine learning models.

\textbf{Related work.} Scaling laws for deep learning
have become well known in recent years \citep{hestness2017deep, rosenfeld2019constructive, kaplan2020scaling, hoffmann2022training}. They serve several purposes, including reasoning about the trade-offs between increasing training data and model parameters \citep{kaplan2020scaling, hoffmann2022training}, predicting the performance of large-scale models \citep{cherti2023reproducible},
and comparing the performance of learning algorithms at manageable scales \citep{dubois2023evaluating}. 
The most similar works to ours are \citet{hashimoto2021model} and \citet{rolf2021representation}, which study how model performance scales when training with multiple data sources; our work instead takes the perspective of studying individual data points, which offers a more granular tool for analyzing the contents of 
a training dataset.

Separately, another line of work about
data valuation focuses
on the role of individual data points in improving the model's performance \citep{ghorbani2019data, kwon2021beta, wang2022data}. These methods typically score training examples based on their \textit{marginal contribution}---how much including them in the training data affects the model's accuracy---by averaging this across many datasets of different sizes. Such methods can be used to identify mislabeled data, filter for high-quality data, upweight helpful examples, and select promising new points for active learning \citep{ghorbani2019data, kwon2023data, jiang2023opendataval}.
Our work is similarly framed around each data point's marginal contribution, but we study how the loss improvement scales with the size of the dataset for each training example, and we find that the most helpful training examples can vary with the dataset size.

Aside from these methods focused on each example's marginal contributions, others analyze individual data points using different importance measures \citep{ilyas2022datamodels, park2023trak, just2023lava}. In particular, many works have explored extensions of the classic influence function to nonlinear models like neural networks \citep{cook1980characterizations, koh2017understanding, grosse2023studying, kwon2023datainf}.
\citet{ilyas2022datamodels} briefly discuss the effect of training dataset size on
example-to-example influence scores, 
but to our knowledge, our work is the first to thoroughly explore scaling behavior for individual data points.

\section{Individualized data scaling laws}
\label{sec:scaling}
We focus here on supervised learning problems, and
as a notational setup, we write individual data points as tuples $z = (x, y)$ where $x$ is an input and $y$ the response variable. We write datasets as $\mcD = \{z_i\}_{i = 1}^k$, and the datasets can have variable sizes that we denote by $|\mcD|$. When creating datasets of different sizes, we assume access to a single large training pool $\mcD_t$ from which the smaller datasets $\mcD$ can be sampled without replacement. We consider a learning algorithm that trains a model on a given dataset $\mcD$, which we denote by $f_{\mcD}$, and we consider several options for the model class (e.g., logistic regression, MLPs). Our analysis focuses on the population error $\mcL(f_\mcD)$, for example the cross-entropy loss, and we only consider the effects of the training data
(i.e., we do not vary the number of model parameters).

Recall the form of current aggregate scaling laws: if we ignore the number of model parameters,
several works are based on the same functional form
$\mcL(f_\mcD) \approx \epsilon + b |\mcD|^{-\lambda}$ \citep{rosenfeld2019constructive, kaplan2020scaling, hashimoto2021model, hoffmann2022training}.
We can view this as a claim about the expected loss given a dataset of size $k$, where datasets are sampled uniformly at random:
\begin{equation}
    \E_{|\mcD| = k}[\mcL(f_\mcD)] \approx \epsilon + \frac{b}{k^\lambda}. \label{eq:aggregate-expectation}
\end{equation}
The functional form in \cref{eq:aggregate-expectation} takes an aggregate view by only focusing on the dataset size $k$, but our goal is to analyze how the loss is affected by
individual data points. We therefore focus on each data point's marginal contribution, which is
defined as the performance difference between models trained with and without $z$.
We write the marginal contribution $\Delta(z, \mcD)$ as follows, where we expect $\Delta(z, \mcD) > 0$ in most cases for useful data points:
\begin{equation}
    \Delta(z, \mcD) = \mcL(f_\mcD) - \mcL(f_{\mcD \cup \{z\}}).
\end{equation}
This quantity is the basis of many data valuation methods, \citep{ghorbani2019data, ghorbani2020distributional, kwon2021beta, wang2022data}, and its expectation across $z$ is reflected via the derivative of current scaling laws like \cref{eq:aggregate-expectation}.\footnote{If the aggregate scaling law $\E_{|\mcD| = k}[\mcL(f_\mcD)] \approx \epsilon + \frac{b}{k^\lambda}$ holds, then we have $\E_{z, |\mcD| = k}[\Delta(z, \mcD)] \approx \frac{b\lambda}{k^{\lambda + 1}}$.}
Our goal here is to understand how marginal contributions scale
as a function of the dataset size $|\mcD|$. We expect them to shrink monotonically towards zero, but at a rate that may be specific to each data point.
We therefore posit the following scaling law for the expectation $\psi_k(z) \equiv \E_{|\mcD| = k}[\Delta(z, \mcD)]$ across datasets of size $k$:
\begin{equation}
    \psi_k(z) \approx \frac{c(z)}{k^{\alpha(z)}}. \label{eq:scaling-law}
\end{equation}
This shares a similar form with aggregate scaling laws, but there is no irreducible term because the contribution should shrink to zero, and we allow both the numerator $c(z)$ and exponent $\alpha(z)$ to depend on the data point.
In the remainder of this section, we show that this functional form holds in practice for several model classes (\Cref{sec:validation}), and we examine support for this parametric form from classical learning settings (\Cref{sec:classical}).

\subsection{Empirical validation of the scaling law} \label{sec:validation}

To test if the scaling law in \cref{eq:scaling-law} holds in practice, we first run a simple experiment to visualize the hypothesized behavior. Our approach is to select a range of dataset sizes $k$ and estimate each expectation $\psi_k(z)$ by averaging a large number of sampled contributions
$\Delta(z, \mcD)$.
We can then plot the results to observe the following linear trend in the log-transformed means and cardinalities $k$,
\begin{equation}
    \log |\psi_k(z)| \approx \log |c(z)| - \alpha(z) \log (k), \label{eq:log-scaling}
\end{equation}
where we use the absolute values $|\psi_k(z)|$ and $|c(z)|$ because certain examples have negative contributions. We conduct this experiment for 1000 data points from the IMDB movie review dataset \citep{maas2011learning}, where we use 10 log-spaced cardinalities between $k = 100$ and $k = 1000$, and we average 1000 samples $\Delta(z, \mcD)$ to estimate each $\psi_k(z)$. 

\begin{figure}[t]
    \centering
    \includegraphics[width=\linewidth,trim={0.2cm 1.5cm 1.5cm 2cm},clip]{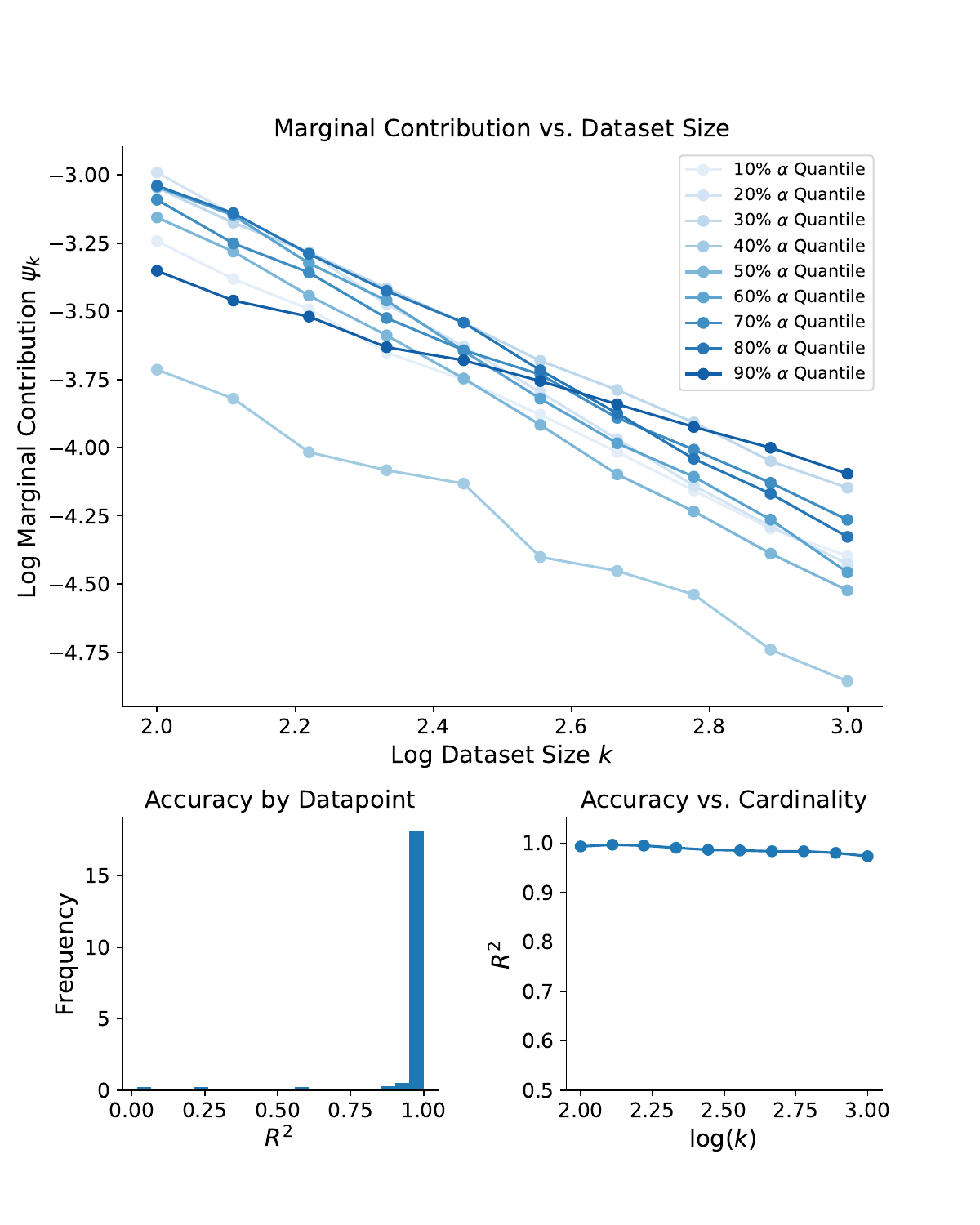}
    \vskip -0.2cm
    \caption{\textbf{Individualized scaling laws for logistic regression trained on the IMDB dataset}. \textbf{Top:} Marginal contribution vs. the dataset size in log-scale for several data points with a range of scaling exponents $\alpha(z)$. \textbf{Left:} Histogram of $R^2$ scores for linear trend lines fit to each data point in the log-scale. \textbf{Right:} Plot of the $R^2$ score from our scaling law predictions at each cardinality, measured across data points. We achieve an overall $R^2 = 0.987$ for the predictions across all points and dataset sizes.}
    \label{fig:validation}
\end{figure}

\begin{table}[hb!]
    \centering
    \caption{\textbf{Overall $R^2$ score of predicting expected contributions}. The overall $R^2$ score is measured for linear trends fit in log-space after estimating $\psi_k(z)$ for a range of $k$ values (the same estimation approach used in \Cref{fig:validation}), and it is calculated across data points and cardinalities.}
    \label{tab:global-r2}
    \vspace{0.2cm}
    \begin{small}
    \begin{tabular}{cccc}
    \toprule
         & Logistic Regression & MLP & SVM \\
    \midrule
    MiniBooNE & 0.993 & 0.963 & 0.929  \\
    CIFAR10 & 0.944 & 0.942 & 0.963  \\
    IMDB & 0.985 & 0.956 &0.946 \\
    \bottomrule
    \end{tabular}
    \end{small}
\end{table}

\Cref{fig:validation} shows results for logistic regression models trained on frozen BERT embeddings \citep{devlin2018bert}. The slope of each line represents the scaling exponent $\alpha(z)$, and the intercept represents the coefficient magnitude $\log |c(z)|$. We observe variability in the estimated slopes, and \Cref{fig:validation} shows several points corresponding to different quantiles of the fitted
scaling exponent $\alpha(z)$.

To test the scaling law's accuracy, we calculate the goodness-of-fit for each data point's linear trend via the $R^2$ score, which we fit with ordinary least squares. \Cref{fig:validation} shows a histogram of these scores, which confirms that the majority of points follow clean
log-linear trends. A minority of points do not, but we find that these outliers generally have very small contributions $\log |\psi_k(z)|$ across $k$ values, which may indicate numerical or optimization noise as the loss differences approach zero (see \Cref{fig:r2_psi}).
To account for the variability between points, we also calculate the $R^2$ from the linear trends within each cardinality, where we see that the score is ${\approx}1.0$ across $k$ values (\Cref{fig:validation} bottom right).
Overall, many data points are accurately described by log-linear trends, but we observe some deviations as $k$ grows and the loss differences shrink towards zero;
such issues may be due to noise, or cases where the scaling law does not perfectly fit the data, but it is accurate enough in most cases to
provide useful insights.

These observations provide initial evidence for the scaling behavior hypothesized in \cref{eq:scaling-law} and motivate the remainder of our analysis.
Because our fitting approach in \Cref{fig:validation} requires a large number of samples for each $z$, we develop a more efficient approach to fit the scaling law in \Cref{sec:estimation}.
In our more comprehensive experiments in \Cref{sec:experiments}, we perform similar analysis of the scaling law's fit using multiple datasets and model classes, and we also discuss implications for data subset selection and data valuation. As a preview, \Cref{tab:global-r2} shows that the scaling law achieves a high overall $R^2$ score
for three datasets and model classes.

\subsection{Theoretical support for the scaling law} \label{sec:classical}
 
We mathematically analyze linear regression and more general M-estimators to provide theoretical support for the existence of the scaling law and give insight into factors that can affect the scaling parameters. The theorems and details are provided in \Cref{app:theory}, but we summarize the results here. Overall, we find general confirmation for a reliable decay rate in the marginal contributions, and we identify intuitive factors that influence the magnitude of the contribution: for example, for linear regression we find that the relative noise of the response variable is relevant, as is the input vector's leverage score. However, our theory does not make precise predictions about the scaling behavior, because it requires strong regularity conditions and has error terms whose expectations may be non-negligible. These limitations
may explain why our empirical results show heterogeneity in the scaling exponent $\alpha(z)$, although it is often concentrated in a range from $[1, 1.5]$ that is consistent with the loose predictions from our theoretical results. We remark that such partial mismatches between theory and practice are also widely observed with aggregate scaling laws \citep{kaplan2020scaling, hutter2021learning}.

\section{Efficient scaling law estimation}
\label{sec:estimation}
A key challenge in using individualized scaling laws is fitting them to every example in a dataset. Recall that \Cref{sec:validation} used a large number of samples $\Delta(z, \mcD)$ at each cardinality to estimate the expectation $\psi_k(z)$.
This lets us visualize
a linear trend in the log-log space, but the procedure is intractable to repeat for every example in a large dataset.
To address this issue, we now describe two approaches to more efficiently estimate individualized scaling laws.

\subsection{Maximum likelihood-based estimation}
\label{sec: likelihood estimation}

Our goal is to fit the scaling parameters $c(z), \alpha(z)$ using a limited number of marginal contribution samples $\Delta(z, \mcD)$, which is challenging because the scaling law in \cref{eq:scaling-law} is defined via the expectation $\psi_k(z)$ at each cardinality. We find that there is significant variability in the marginal contributions at each $k$; in fact, it is common to observe a mix of positive and negative contributions for a single example $z$ (see \Cref{fig:parametric-example}). This means that whereas our procedure in \Cref{sec:scaling} relied on noiseless log-transformed mean contributions $\psi_k(z)$, a more efficient procedure must tolerate a noisy view of the scaling trend.

We therefore propose fitting the scaling law in the original loss space, and in a way that carefully handles noise in the sampled contributions, which tends to be largest at smaller cardinalities $k$. We frame this as a statistical inference problem by modeling the marginal contributions at each $k$ with the following Gaussian distribution:
\begin{equation}
    \Delta(z, \mcD) \mid |\mcD| = k \sim \mcN\left(\frac{c(z)}{k^{\alpha(z)}}, \frac{\sigma^2(z)}{k^{\beta(z)}}\right). \label{eq:gaussian}
\end{equation}
Notice that the mean for each $k$ is governed by our scaling law from \cref{eq:scaling-law}, and the variance is governed by a separate scaling law with parameters $\sigma(z)$ and $\beta(z)$. The variance in marginal contributions also naturally shrinks with $k$ \citep{kwon2021beta}, and we observe
that the variance follows this general
scaling behavior (see \Cref{fig:r2 var} in \Cref{appendix: evidence}). Fitting the variance parameter $\sigma(z)$ and exponent $\beta(z) > 0$ lets us capture the natural reduction in variability as $k$ grows, and it encourages the learned scaling law to tolerate larger deviations for small $k$ values. This is essential for correctly modeling the mean-related parameters $c(z), \alpha(z)$, but we do not otherwise use the variance-related parameters.\footnote{When fitting their scaling law, \citet{hoffmann2022training} handled outliers by using a Huber loss and downweighting points in the low-compute regime. Directly modeling the variance is an alternative approach that we have not seen in the scaling law literature.}

Based on this statistical model, we can fit the parameters by obtaining a set of samples $\Delta(z, \mcD_i)$ for $i = 1, \ldots, m$ and then minimizing the negative log-likelihood as follows:
\begin{equation}
    \argmin_{c, \alpha, \sigma, \beta} \; \frac{1}{m} \sum_{i = 1}^m \mathrm{NLL}(\Delta(z, \mcD_i), |\mcD_i|; c, \alpha, \sigma, \beta). \label{eq:gaussian-objective}
\end{equation}
The optimization problem is non-convex, but we find that it can be effectively optimized with Adam \citep{kingma2014adam} using a moderate number of gradient steps; our exact procedure is described in \Cref{app:parametric-estimator}, and it involves using analytic solutions for $c, \sigma$ to only take gradient steps on $\alpha, \beta$. An example of the fitting result is shown in \Cref{fig:parametric-example}, where we see that the learned curve correctly allows larger deviations at smaller $k$ values and fits more precisely when $k$ is large. Our experiments use up to $m = 1000$ samples per data point for the most accurate results, which is an order of magnitude fewer samples than the validation in \Cref{sec:scaling} (which used $1000$ samples \textit{per cardinality}). Our experiments also test the estimator's convergence and show
strong results with as few as $m = 100$ samples (see \Cref{sec:experiments}), which provides a $100\times$ speedup.

\begin{figure}[t]
    \centering
    \includegraphics[width=\linewidth,trim={0.0cm 0.5cm 0.0cm 0.4cm},clip]{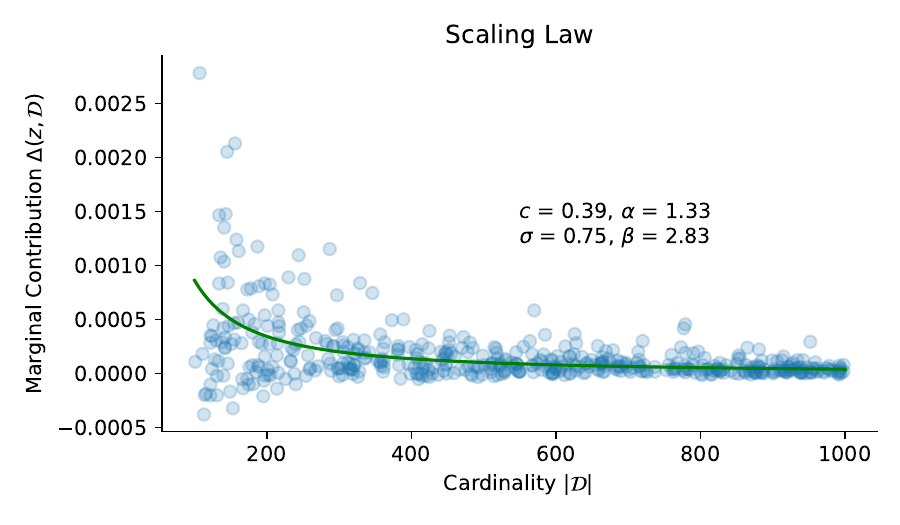}
    \vskip -0.2cm
    \caption{\textbf{Likelihood-based scaling law estimator.} The estimator is fit for a single example from the IMDB dataset, and the scaling parameters are fit using $m = 100$ samples.}
    \label{fig:parametric-example}
    \vskip -0.1cm
\end{figure}

\subsection{Amortized scaling law estimation} \label{sec:amortization}

The likelihood-based approach is
a more efficient way to fit individualized scaling laws, but the number of samples $m$ required for each example $z$ is still moderately large (e.g., $m = 100$). The key issue is that each data point is treated independently, and therefore requires its own set of samples $\Delta(z, \mcD)$ to fit the scaling curve. We now consider an alternative approach where we share information across examples $z$ to use fewer samples per data point, and in the process jointly fit the scaling parameters $c(z), \alpha(z)$ for an entire dataset.

Our approach
is to replace the objective from \cref{eq:gaussian-objective}, which optimizes over the scaling parameters for a single $z$,
with a combined objective parameterized by a neural network.
We use a network $g(z; \theta) \in \R^4$ that predicts the four parameters for the Gaussian distribution in \cref{eq:gaussian}: the network must receive only the data point $z = (x, y)$ as an input, which we implement by passing $x$ to a standard architecture (e.g., a MLP), predicting scaling parameters for all classes, and using only those for the relevant class $y$.
We then train the network using the following loss function:
\begin{equation}
    \min_{\theta} \; \E_{z, \mcD} \Big[\mathrm{NLL}\big(\Delta(z, \mcD), |\mcD|; g(z; \theta)\big)\Big].
\end{equation}
Like the previous objective, this is optimized using a set of samples $\Delta(z, \mcD)$, but
this version is equivalent to solving \cref{eq:gaussian-objective} simultaneously for all data points with parameters
predicted by the network $g(z; \theta)$. This general approach is known as \textit{amortized optimization} and is used for a range of machine learning tasks \citep{amos2023tutorial}, but not previously for estimating scaling laws. We can in principle train the network with as few as $m = 1$ contributions per data point,
but in our experiments we construct a dataset of relatively few samples $\Delta(z, \mcD)$ per data point (e.g., $m = 10$), and we train
until the loss stops improving on a held-out validation set.

\section{Experiments}
\label{sec:experiments}
\subsection{Evidence for individualized data scaling laws}
\label{sec: evidence}

We begin our more comprehensive experiments by providing further empirical evidence to support the parametric scaling law in \cref{eq:scaling-law}.\footnote{Code for our experiments is available at \url{https://github.com/iancovert/data-scaling}}
We focus on three model classes: logistic regression, SVMs and MLPs (specifically, two-layer ReLU networks).
We evaluate these models on three different datasets: MiniBooNE \cite{roe2005boosted}, CIFAR-10 \cite{krizhevsky2009learning}, and IMDB movie reviews \cite{maas2011learning}. To accelerate training and avoid underfitting,
we use frozen ResNet-50 \cite{he2016deep} and BERT \cite{devlin2018bert} pre-trained embeddings for CIFAR-10 and IMDB, respectively. Full details for the experiments are provided in \Cref{appendix: evidence}.

For each setting mentioned above, we evaluate the model performance by computing the cross-entropy loss on a test dataset of size 1000. Similar to our analysis in \Cref{sec:validation}, we first test the accuracy of log-linear trendlines by averaging many samples $\Delta(z, \mcD)$ at a range of log-spaced cardinalities. For logistic regression, we use 1000 data points and 1000 samples per $k$ value. For SVMs and MLPs, due to the higher variance in marginal contributions, we use 200 data points and 5000 samples per dataset size $k$.

Following \Cref{fig:validation},
we show the $R^2$ from lines fit to each point, and we observe across settings that the $R^2$ scores are generally very close to 1. \Cref{fig: r2} shows these results via histograms for each setting. The mass is most concentrated towards 1 for logistic regression, whereas MLPs and SVMs have more points with sub-optimal fits. We attribute this to worse estimation of the mean at each cardinality, and a similar issue mentioned in \Cref{sec:validation} with instability as the loss differences approach zero. The $R^2$ across data points and cardinalities is reported in \Cref{tab:global-r2}, which reveals that the scaling law overall captures nearly all the variability in expected contributions at each dataset size, despite the existence of some poorly fitting data points.

\begin{figure}[t]
    \centering
    \includegraphics[width=\linewidth]{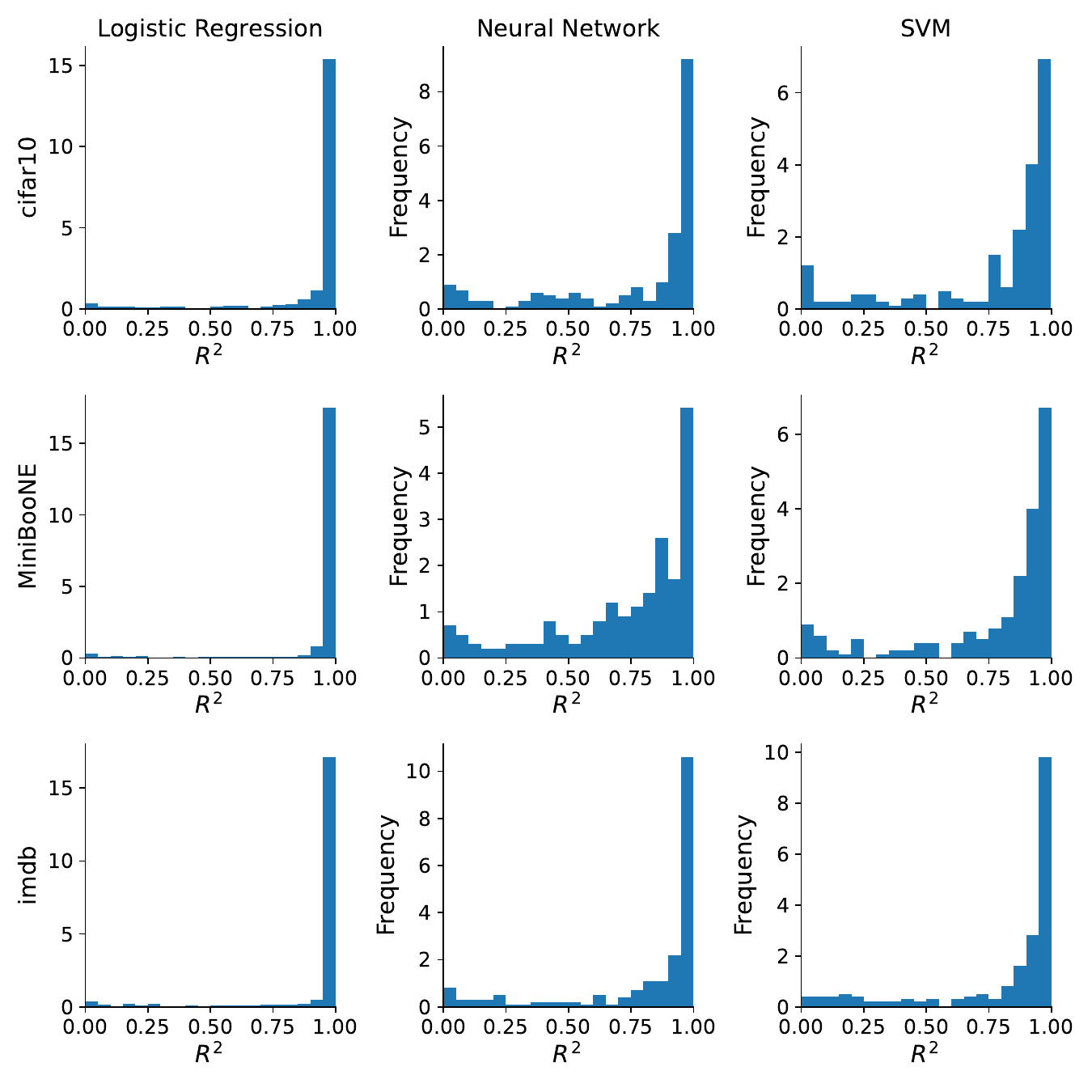}
    \vskip -0.3cm
    \caption{\textbf{Histogram of $R^2$ score when fitting the scaling law.} Similar to \Cref{fig:validation}, we find that linear trends in log-space achieve high $R^2$ scores for most data points, which supports the parametric form in \cref{eq:scaling-law}.
    }
    \label{fig: r2}
    \vskip -0.3cm
\end{figure}

Next, we examine the distribution of estimated parameters ${\alpha}(z)$. Analysis of ${c}(z)$ is deferred to \Cref{appendix: experiments}. Histograms of the fitted values are shown in \Cref{fig: alpha}, where we see that the distribution of ${\alpha}(z)$ has a mode between 1 and 1.5, which coincides with the range seen in our theoretical results (see \Cref{app:theory}).
The heterogeneity in $\alpha(z)$ is important, because it suggests that
different data points have
different decay rates for their marginal contributions. A data point with a higher $\alpha(z)$ value may have a larger contribution in a small-sample regime compared with other data points, but due to its fast decay rate, it could have a much smaller contribution in a large-sample regime. We illustrate this phenomenon in \Cref{fig: lines}: for each setting, we select the points with $\alpha(z)$ in 20\%, 40\%, 60\%, and 80\% quantiles among all points being evaluated, and we show their marginal contribution against the size of the preceding dataset. The results illustrate that there exist many lines that cross one another, which supports the idea that the most valuable data points depend on the dataset size.
In \Cref{sec: application subset}, we test this insight in experiments where we add points to an existing dataset based the scaling law's predictions at different $k$ values.

\begin{figure}[ht]
    \centering
    \includegraphics[width=\linewidth]{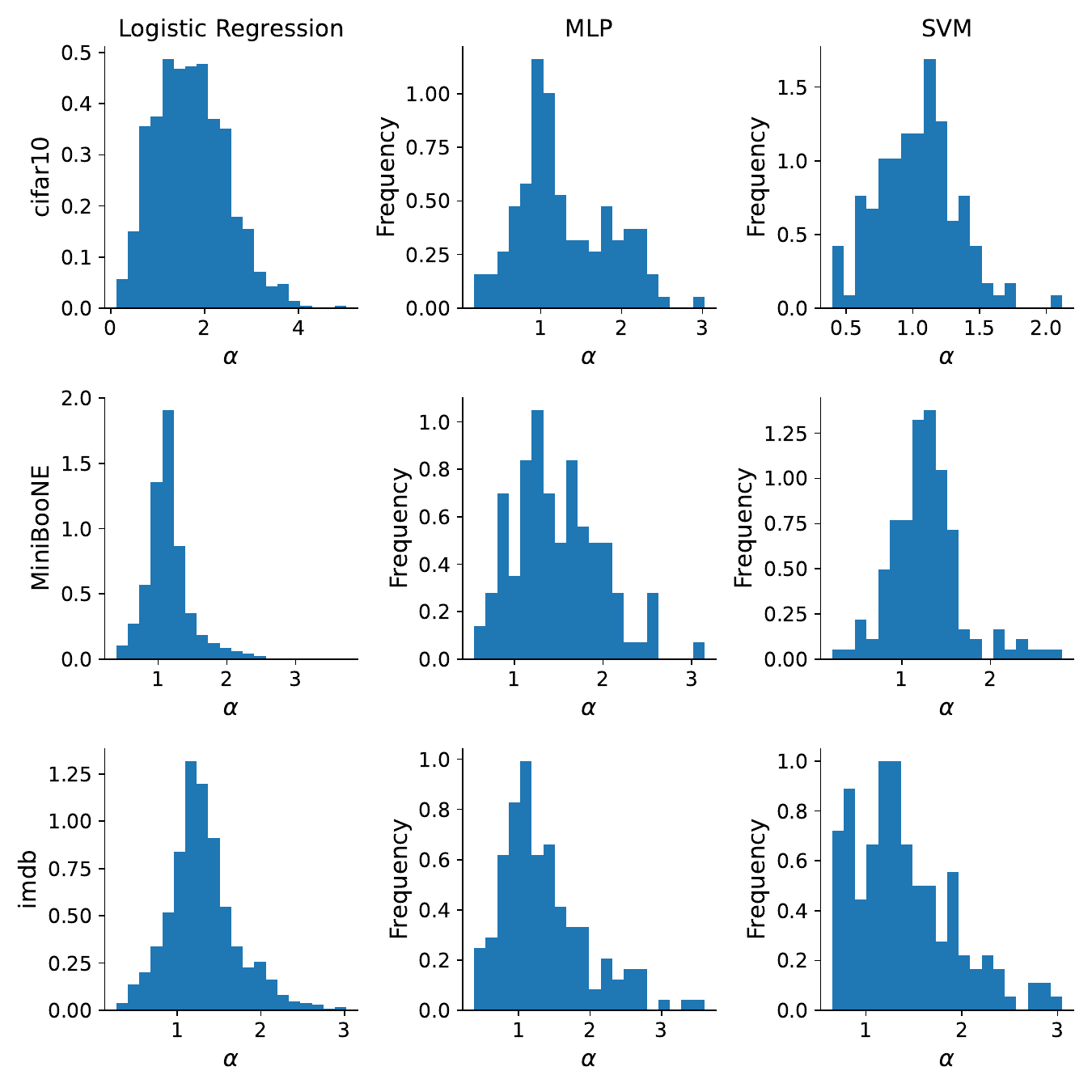}
    \vskip -0.3cm
    \caption{\textbf{Histogram of estimated $\alpha(z)$}. The estimated values have a mode between 1 and 1.5 and exhibit significant heterogeneity. We exclude points with $R^2<0.8$ to ensure the estimated $\alpha(z)$ values are reliably estimated.}
    \label{fig: alpha}
\end{figure}

\begin{figure}[ht]
    \centering
    \subfigure[CIFAR10 with MLP]{\includegraphics[width=0.49\linewidth]{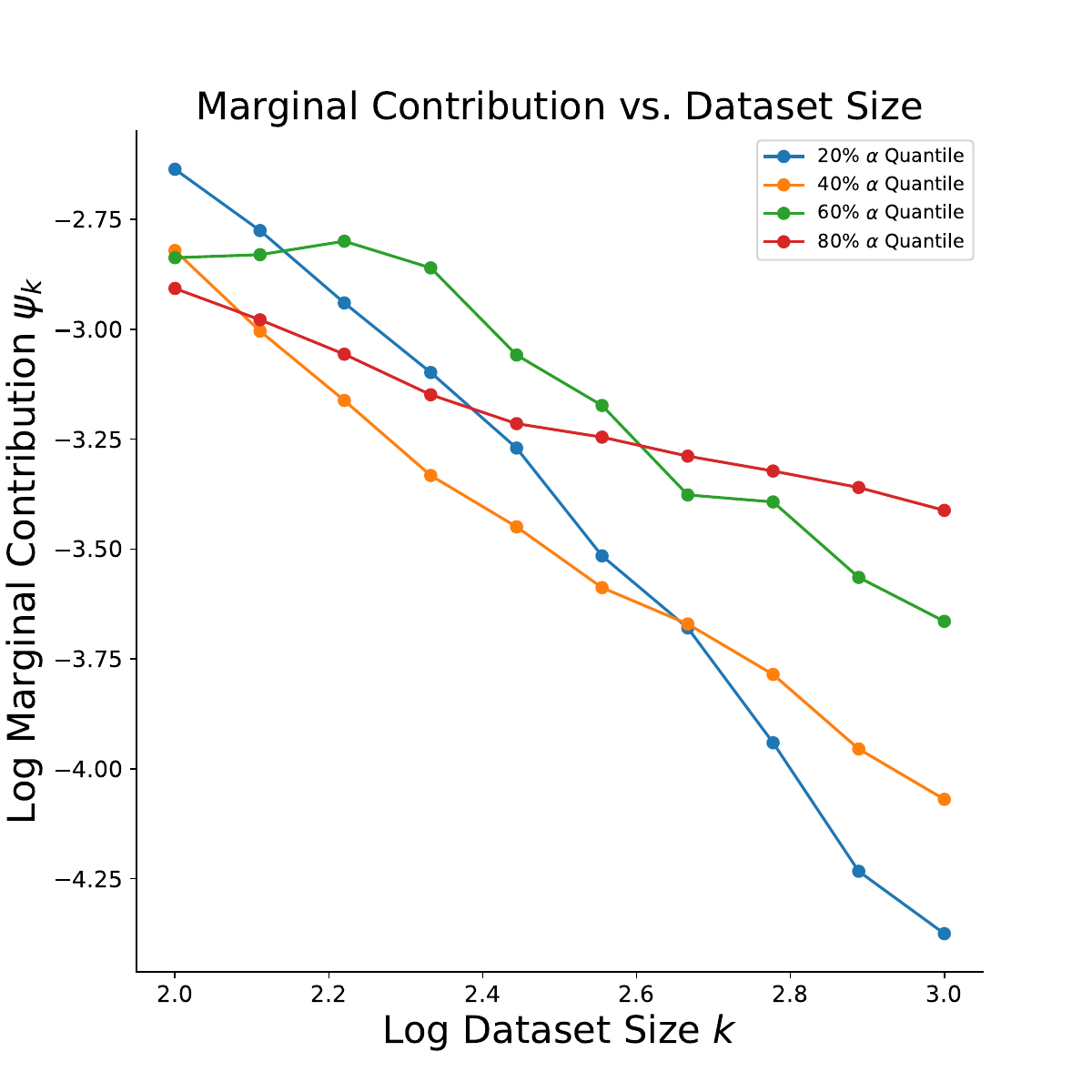}}
    \subfigure[IMDB with Log Reg]{\includegraphics[width=0.49\linewidth]{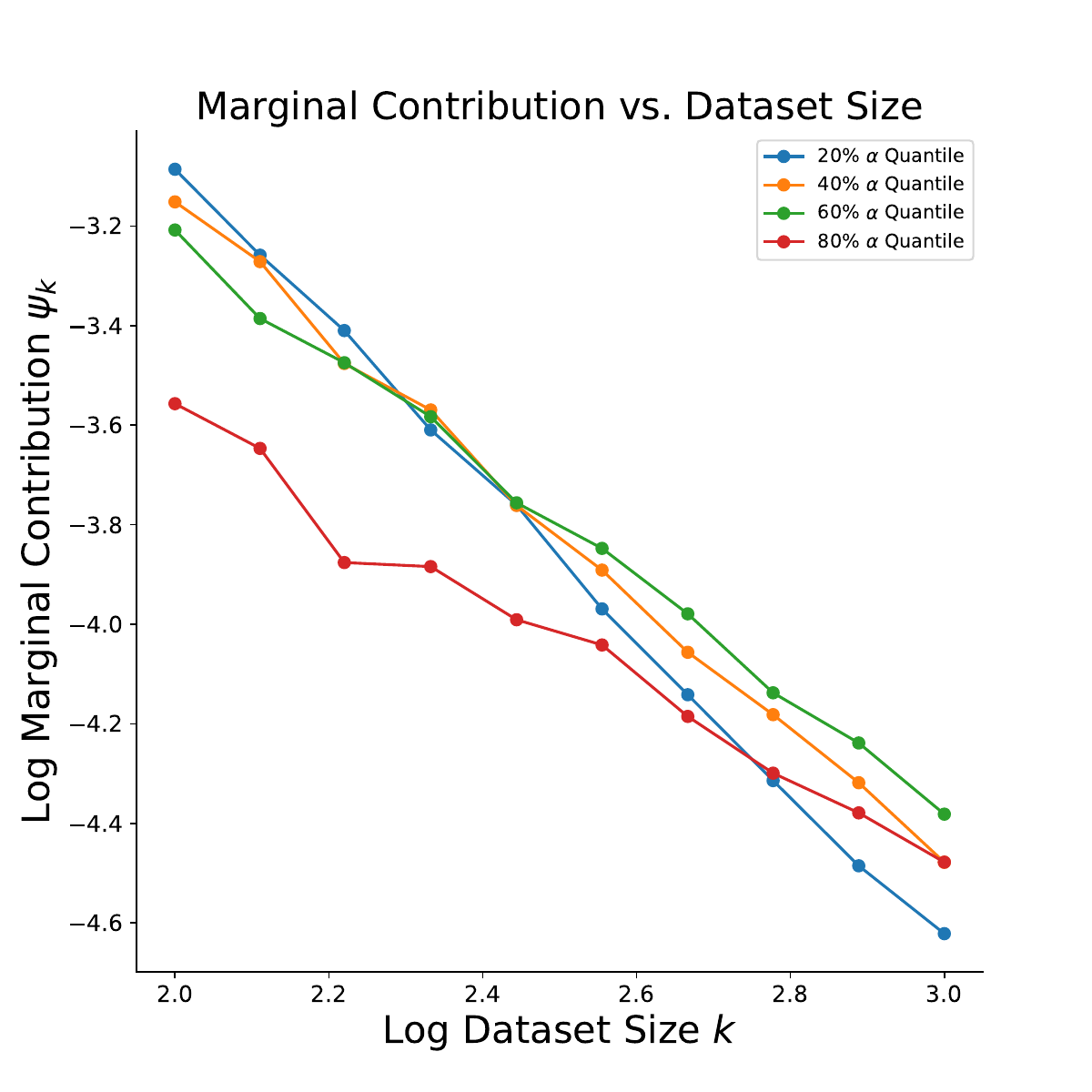}}
    \caption{\textbf{Marginal contribution for points with different $\alpha(z)$}. Similar to \Cref{fig:validation}, we plot the expected contribution $\log |\psi_k(z)|$ against the dataset size $\log k$.
    Lines with different $\alpha(z)$ cross one another, indicating that the ranking of valuable points depends on the dataset size $k$.
    }
    \label{fig: lines}
\end{figure}

A natural question to ask is why different points have different scaling exponents $\alpha(z)$. A comprehensive
answer to this question is beyond the scope of our work, but we provide some preliminary ideas.
In \Cref{fig: dist}, we show the relationship between the estimated $\alpha(z)$ and the distance to the decision boundary for logistic regression models with two binary classification datasets. There is a strong correlation between
$\alpha(z)$ and the distance to the decision boundary: closer points have lower $\alpha(z)$ and further points have higher $\alpha(z)$. A possible explanation is that
when the size of the preceding dataset is small, every point will have a significant contribution to learning the decision boundary
in a high-dimensional space. However, when the
dataset is large, the logistic regression model may have an accurate decision boundary, and only nearby points contribute by helping better calibrate the magnitude of the normal vector; these points may be those with a slower decay rate.
However, we also remark that $\alpha(z)$ is not completely determined by this factor, and that there is significant heterogeneity very near the decision boundary, so there are other factors at play.
We leave a more thorough analysis of this behavior as a future research direction.

Finally, we also examine the similarity between estimated scaling parameters for different model types. \Cref{tab:alpha-corr} shows the Pearson correlation between $\alpha(z)$ estimates for each pair of models, and we observe a positive correlation for two out of three datasets, CIFAR-10 and IMDB, but relatively weak correlations for MiniBooNE. The degree of correlation between models is similar for the estimated $c(z)$ values. The deviations between models may be attributed to the differing geometry of each function class, where different points can be more influential depending on the shape of the decision boundary. Again, we leave a more complete characterization of this finding to future work.

\begin{table}[h!]
    \centering
    \caption{\textbf{Correlation between estimated $\alpha(z)$ for different models.} The Pearson correlation between estimated scaling exponents shows weak positive similarity for CIFAR-10 and IMDB.}
    \vspace{0.2cm}
    \begin{small}
    \begin{tabular}{cccc}
         \toprule
          & LogReg/MLP & LogReg/SVM & SVM/MLP \\
         \midrule
         MiniBooNE & -0.01 & 0.18 & -0.04 \\
         CIFAR-10  & 0.38 & 0.26 & 0.34 \\
         IMDB  & 0.49 & 0.50 & 0.27 \\
         \bottomrule
    \end{tabular}
    \end{small}
    \label{tab:alpha-corr}
\end{table}

\begin{figure}[ht]
    \centering
    \subfigure[MiniBooNE]{\includegraphics[width=0.49\linewidth]{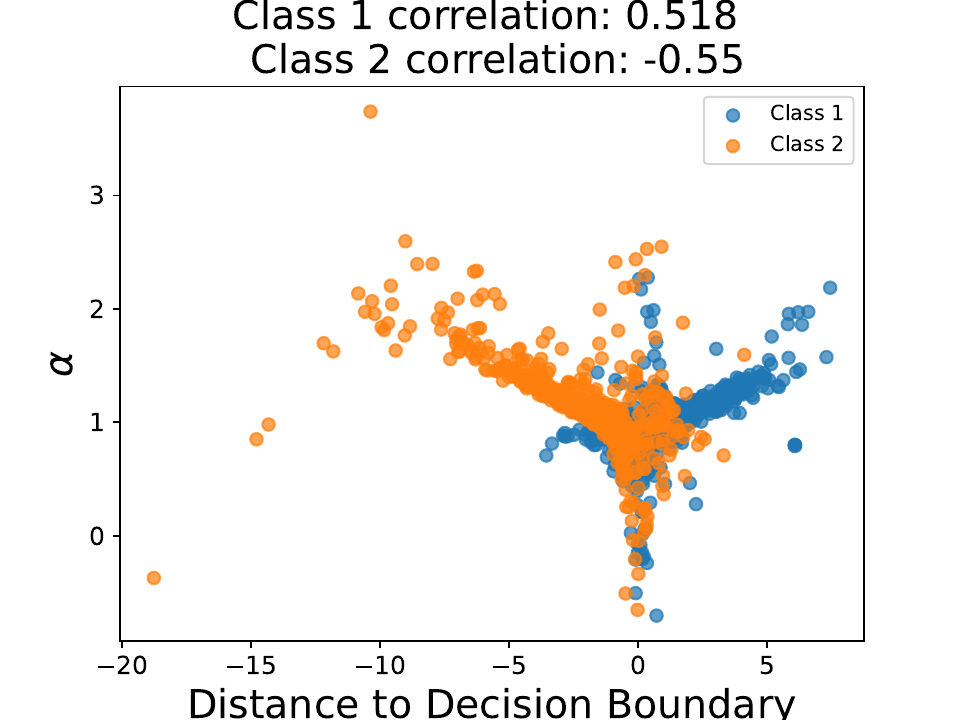}}
    \subfigure[IMDB]{\includegraphics[width=0.49\linewidth]{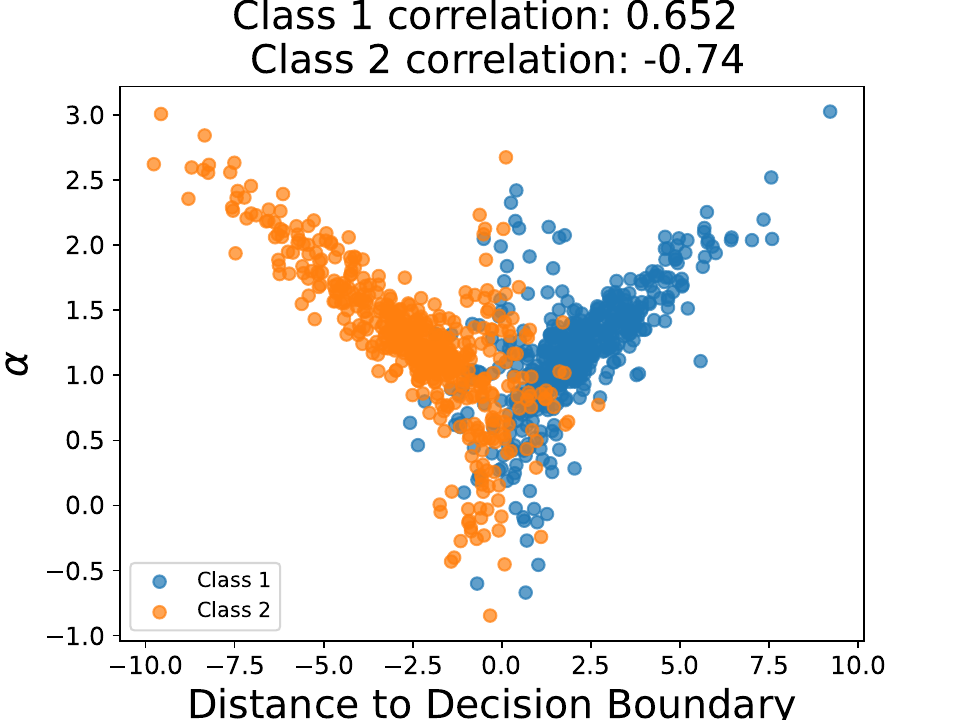}}
    \caption{\textbf{Distance to decision boundary vs. scaling exponent for logistic regression model.} We fit a logistic regression model on all points being evaluated and then compute the distance of each example to the decision boundary. The sign of distance is kept to distinguish points from each class. There is a strong correlation between distances and the estimated $\alpha$ values.}
    \label{fig: dist}
\end{figure}

\subsection{Scaling law estimation accuracy}
\label{sec:estimation-accuracy}

We now move to testing our more efficient scaling law estimators from \Cref{sec:estimation}. In evaluating our proposed approaches,
we use metrics based on how accurately we predict the marginal contributions at each dataset size. For example, \Cref{fig:estimator-scatter} shows that for IMDB with logistic regression, we can accurately predict the expectation $\psi_k(z)$ at multiple dataset sizes from $k = 100$ to $k = 1000$.
For a more systematic evaluation, \Cref{fig:estimator-accuracy} shows the accuracy of our scaling law predictions across dataset sizes for both versions of our likelihood-based estimator,
and when fitting with different numbers of samples. Our scaling laws are all fit in the range $k \in [100, 1000]$, so we mark the boundary where our scaling law begins extrapolating to larger ranges. An expanded version of these results is shown in \Cref{app:estimator-accuracy}, where we find that the $R^2$ score drops when we extrapolate beyond $k = 2500$, but the correlation and rank correlation
with the true expectation remain high.

\begin{figure}[h]
    \centering
    \includegraphics[width=0.8\linewidth,trim={13.0cm 0.0cm 13.0cm 0.0cm},clip]{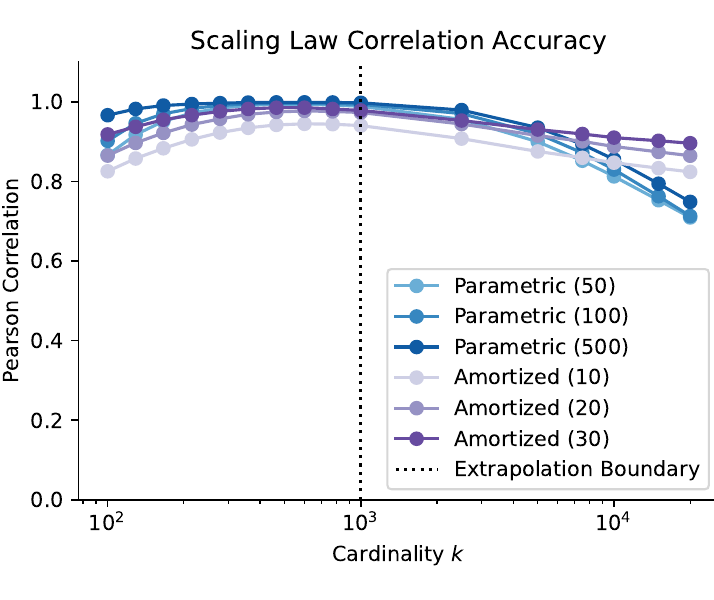}
    \vskip -0.4cm
    \caption{\textbf{Scaling law estimation accuracy at multiple cardinalities for IMDB with logistic regression.} We evaluate the scaling law's predictions at cardinalities up to an order of magnitude larger than where it is fit, according to the Pearson correlation with the true expectation $\psi_k(z)$ for each $k$. The numbers in parentheses are the number of samples used to fit the estimator.}
    \label{fig:estimator-accuracy}
\end{figure}

The likelihood-based estimator improves as we fit with more samples, and it provides the highest accuracy across all $k$ value when fit with 500 samples. The amortized estimator is relatively accurate given the small number of samples. The accuracy of all the estimators are consistently high in the fitting range $k \in [100, 1000]$, although the performance suffers slightly at the dataset sizes with
the noisiest samples.
Interestingly, we see that as we extrapolate to datasets an order of magnitude larger than our fitting range, the amortized estimator tends to be more accurate than the parametric estimator even though it is fit with
many fewer samples.

\subsection{Application to data valuation}
\label{sec:valuation}

Next, we consider the application of our scaling laws to data valuation.
Data valuation is the most direct application of our scaling analysis, because these methods are typically based on a weighted average of the expectations $\psi_k(z)$ at different $k$ values \citep{ghorbani2019data, kwon2021beta, wang2022data}, which means we can estimate the scores using our scaling parameters $c(z), \alpha(z)$. In fact, our approach can be viewed as a version of a stratified data valuation estimator \citep{wu2023variance} with the mean at each $k$ estimated via the shared set of scaling parameters.
Among the various available methods, we focus on Distributional Data Shapley \citep{ghorbani2020distributional}, which uses a uniform average across cardinalities. We measure the estimation accuracy relative to ground truth values computed using a large number of sampled marginal contributions, and we compare the accuracy to a standard Monte Carlo estimator (see a more thorough discussion in \Cref{app:valuation}). 
Our goal is to show that our estimators converge to the same value as the conventional estimator, which attests to the scaling law accurately capturing the decay in marginal contributions; and we hope to observe that one or both of our estimators can outperform the standard
one when using a small number of marginal contribution samples.

\begin{table}[t]
    \centering
    \caption{\textbf{Data valuation accuracy with 10 samples per data point.} We report the Pearson correlation of the estimates with ground truth values. Note that the non-amortized likelihood estimator can be unreliable with insufficient samples.}
    \vspace{0.2cm}
    \begin{small}
    \begin{tabular}{cccc}
         \toprule
         Method & MiniBooNE & CIFAR-10 & IMDB \\
         \midrule
         Monte Carlo & 0.967 & 0.857 & 0.916 \\
         Likelihood  & 0.858 & 0.077 & 0.342 \\
         Amortized   & 0.994 & 0.893 & 0.930 \\
         \bottomrule
    \end{tabular}
    \end{small}
    \label{tab:valuation-small}
    \vspace{-0.2cm}
\end{table}

\begin{table}[t!]
    \centering
    \caption{\textbf{Data valuation accuracy with 50 samples per data point.} We report the Pearson correlation of the estimates with ground truth values. All the estimators are quite accurate.}
    \vspace{0.2cm}
    \begin{small}
    \begin{tabular}{cccc}
         \toprule
         Method & MiniBooNE & CIFAR-10 & IMDB \\
         \midrule
         Monte Carlo & 0.988 & 0.965 & 0.981 \\
         Likelihood  & 0.995 & 0.936 & 0.985 \\
         Amortized   & 0.996 & 0.931 & 0.986 \\
         \bottomrule
    \end{tabular}
    \end{small}
    \label{tab:valuation-large}
\end{table}

Our full results are in \Cref{app:valuation}, but a summary is shown in \Cref{tab:valuation-small} and \Cref{tab:valuation-large} with results calculated for estimates using 10 and 50 samples, respectively. In general, we observe that the amortized estimator achieves the best Pearson correlation with the ground truth valuation scores when using just 10 samples; this reflects its improved sample complexity from sharing information across data points. In contrast, the parametric estimator cannot fit reliably with so few samples.
The three estimators tend to perform similarly with larger numbers of samples, with the exception of CIFAR-10, where the Monte Carlo estimator works better. Overall, this validates that our scaling law accurately captures the decay in marginal contributions with $k$, and offers an alternative approach to calculate data valuation scores.

\subsection{Application to point addition}
\label{sec: application subset}

We now apply our scaling law to select high-value examples to add to an existing dataset.
Based on our observation in \Cref{fig: lines}, the lines
for different points can cross one another, indicating that the relative rankings of data points may vary depending on the size of the
dataset. This suggests that
the size of the preceding dataset should be taken into account when evaluating which point is most helpful to add.
We therefore design a set of experiments inspired by \citet{kwon2021beta} where we compare the addition of high-value data points selected by different methods.

Here, we consider selecting 20 points from a population of 1000 to be added to preceding datasets with two different sizes: 100 and 1000. Each preceding dataset is randomly sampled from the whole training dataset, excluding the pool from which we select the new points, which is also randomly sampled. We use two baseline approaches to select points, random selections from the pool and Distributional Shapley values \citep{ghorbani2020distributional} fit with a large number of marginal contribution samples. For our approach, we fit the scaling parameters using the likelihood estimator from \Cref{sec: likelihood estimation} with 500 samples, and we use the parameters to estimate the expected contribution $\psi_k(z) \approx c(z) / k^{\alpha(z)}$. Because we consider two initial dataset sizes, we test our scaling law with two different $k$ values, $k = 100$ and $k = 1000$, which we refer to as \textit{Scaling 100} and \textit{Scaling 1000}.

\begin{table}[t]
    \centering
    \caption{\textbf{Accuracy improvement (\%) with 20 points added to preceding datasets of size 100.} Each element represents the improvement in test set accuracy with and without the selected points. The last row reports the model's average accuracy before adding new points.
    }
    \vspace{0.2cm}
    \begin{small}
    \begin{tabular}{cccc}
         \toprule
         Method& MiniBooNE & CIFAR-10 & IMDB \\
         \midrule
        Scaling 1000 & 0.21 $\pm$ 1.03      & 1.99 $\pm$ 0.97     & 0.88 $\pm$ 0.96     \\
        Scaling 100  & \textbf{0.47 $\pm$ 1.02}     & \textbf{3.28 $\pm$ 1.09 }    & \textbf{2.22 $\pm$ 0.86}    \\
        Shapley       & 0.28 $\pm$ 1.02      & 2.06 $\pm$ 0.92     & 0.77 $\pm$ 1.20     \\
        Random    & 0.44 $\pm$ 0.91      & 2.01 $\pm$ 1.09     & 1.00 $\pm$ 0.85     \\
        \midrule
        Preceding  & 80.76 $\pm$ 1.58    & 75.50 $\pm$ 1.97    & 78.77 $\pm$ 1.64    \\
         \bottomrule
    \end{tabular}
    \end{small}
    \vspace{-0.3cm}
    \label{tab: subset small}
\end{table}

\begin{table}[t]
    \centering
    \caption{\textbf{Accuracy improvement (\%) with 20 points added to preceding datasets of size 1000.} The settings are the same as \Cref{tab: subset small}.
    }
    \vspace{0.2cm}
    \begin{small}
    \begin{tabular}{cccc}
         \toprule
         Method& MiniBooNE & CIFAR-10 & IMDB \\
         \midrule
        Scaling 1000 & 0.20 $\pm$ 0.39      & 0.17 $\pm$ 0.14     & \textbf{0.13 $\pm$ 0.10}     \\
        Scaling 100  & 0.21 $\pm$ 0.34      & 0.05 $\pm$ 0.06     & 0.06 $\pm$ 0.05     \\
        Shapley       & \textbf{0.22 $\pm$ 0.36}      & \textbf{0.18 $\pm$ 0.14}     & 0.11 $\pm$ 0.09     \\
        Random    & 0.03 $\pm$ 0.20      & 0.04 $\pm$ 0.13     & 0.03 $\pm$ 0.10     \\
        \midrule
        Preceding  & 83.92 $\pm$ 0.46    & 85.38 $\pm$ 0.40    & 85.45 $\pm$ 0.35    \\
         \bottomrule
    \end{tabular}
    \end{small}
    \label{tab: subset large}
\end{table}

The results are shown in \Cref{tab: subset small,tab: subset large}, where we repeat the experiment 1000 times to report the mean and standard deviation in accuracy improvement. According to the results, Scaling 100 achieves the best performance when $k=100$, as it selects based on the predicted contribution for this dataset size.
However, the same set of points performs badly for $k=1000$, indicating that points with high contribution in the small-sample regime and large-sample regime are different. Similarly, we observe that Scaling 1000 performs well when $k=1000$, but badly when $k=100$. Interestingly, we observe that the performance of Shapley and Scaling 1000 are very similar, implying that the Shapley values are closely correlated with
contributions in the large sample regime, and may
be unreliable when evaluating performance in a small-sample regime.
In \Cref{app:point-addition}, we conduct a similar experiment with larger dataset sizes, which demonstrates the same trend.

In conclusion, selecting new training examples using the scaling law predictions at different dataset sizes can lead to quite different results, and selecting points based on the current dataset size
can help more reliably
achieve better performance.
This provides further support for our analysis of scaling behavior for individual data points, but we caution that a more powerful data selection method should take into account not just the size but the contents of the current dataset; our approach only quantifies the expected improvement when a data point is added to a \textit{random} dataset, and this is a limitation to its current usage for dataset selection.

\section{Discussion}
\label{sec:discussion}
This work studies scaling laws for the value of individual data points, and we found evidence for a simple parametric form that holds across datasets and model classes. Our experiments validated the scaling law by visualizing the log-linear trend and testing the scaling law's ability to predict
marginal contributions at different dataset sizes, and we found that the scaling law can extrapolate to larger dataset sizes than the range where it is fit. The behavior we study is costly to measure for an entire training dataset, so we also developed procedures to estimate the scaling parameters from a small number of noisy observations per data point.

Our study is a first step towards understanding this phenomenon, and it has certain limitations and suggests several future directions.
Fitting individual scaling laws remains somewhat expensive, as we require separate marginal contribution samples for each point; it may be possible to borrow insights from recent data valuation estimators that avoid calculating marginal contributions for each point and rely on a single pool of shared models \citep{feldman2020neural, wang2022data, li2023robust}. Another limitation is that our method is not specifically designed for dataset selection, and selecting many highly ranked points may lead to negative interaction effects that our scaling law does not take into account; we speculate that our study of scaling behavior could also be useful for predicting contributions to a specific dataset. Our analysis is restricted to the effect of individual points on a model's error, whereas other recent works have considered how training examples impact specific predictions \citep{ilyas2022datamodels, park2023trak, grosse2023studying};  a follow-up direction
is to analyze whether similar scaling behavior holds for these datapoint-to-datapoint relationships.
Our experiments additionally focused on relatively simple models and small-scale datasets, so an important direction for future work is studying the same phenomenon for large-scale deep learning models.
Finally, we believe that similar techniques that we use here can also be used to study how the value of a fixed \textit{subset} of data scales as the rest of the dataset increases; this would be a natural interpolation of individualized and aggregate scaling.

\section*{Impact Statement}
This paper presents work whose goal is to advance the field of machine learning. There are many potential societal consequences of our work, none which we feel must be specifically highlighted here.

\section*{Acknowledgements}
We thank members of the Zou and Hashimoto labs for helpful discussions, and the anonymous reviewers for their feedback. JZ is supported by funding from the CZ Biohub and NSF CAREER 1942926. 

\bibliography{main}
\bibliographystyle{apalike}

\clearpage
\appendix
\onecolumn
\section{Theoretical support for the individualized scaling law} \label{app:theory}

We now consider theoretical perspectives on our scaling law for some simple models and data distributions. We aim to provide preliminary mathematical support for the existence of the scaling law and give insight into factors that can affect the scaling parameters.
A precise characterization
requires reasoning about the expectation over $\Delta(z, \mcD)$ for datasets of a fixed size, which we find is a significant technical obstacle.
We instead characterize the limiting behavior of $\Delta(z, \mcD)$, which provides some general confirmation of the scaling behavior we observe in practice,
and we leave a more rigorous analysis of $\psi_k(z)$ as a future direction.


We begin by considering the case of linear regression, where datasets $\mcD = \{z_i\}_{i=1}^k$ are generated from the following distribution,
\begin{equation*}
\label{eq:linear}
    (\mcP): y_i = x_i^\top \beta^* + \epsilon_i,  x_i \iid \mcN(0, \Sigma), \epsilon_i\iid \mcN(0,\sigma^2),
\end{equation*} 
with $\Sigma\in\R^{d\times d}$ assumed to be positive definite. Denote the data point to be added as $z = (x, y)$, the preceding data matrix as $\bX_\mcD = (x_1, \ldots, x_k)^\top \in \R^{n\times d}$, and the ordinary least squares estimator corresponding to $\mcD$ and $\mcD\cup\{z\}$ as
$\hat\beta_{\mcD}$ and $\hat\beta_{\mcD\cup\{z\}}$. The loss function is defined as the mean squared error on the population,
\begin{equation*}
    \mcL(\beta) = \E_{(x',y')\sim\mcP}[(y' - \beta^\top x')^2],
\end{equation*}
and the marginal contribution can be written as $\Delta(z, \mcD) = \mcL(\hat\beta_{\mcD}) - \mcL(\hat\beta_{\mcD\cup\{z\}}).$
With this setup, we can characterize the marginal contribution of $z$ as follows. 
\begin{theorem}
\label{thm:linear}
If we denote the noise in $z$ as $\epsilon = y - x^\top\beta^*$, we have the following expectation with respect to the labels conditioned on the preceding dataset $\bX_\mcD$:
    $$\E[\Delta(z,\mcD) \mid z,\bX_{\mcD}] = \frac{2\sigma^2 - \epsilon^2}{k^2}x^\top\Sigma^{-1}x+o_P\left(\frac{1}{k^2}\right).$$
\end{theorem}
The proof
is deferred to \Cref{appendix: linear model}.
    According to \Cref{thm:linear}, marginal contributions are expected to follow the scaling law with a scaling exponent $\alpha(z) = 2$ in the large sample regime.
    Meanwhile, the coefficient $c(z)$ depends on two factors: the first term $2\sigma^2 - \epsilon^2$ represents the relative noise level of $z$, where less noise
    results in a higher contribution, and the threshold for positive/negative contribution is twice the population noise level $\sigma^2$. The second term $x^\top\Sigma^{-1}x$ represents the asymptotic leverage score of the point, where a point with a higher score is more influential, and hence amplifies the magnitude of the contribution. Here we only analyze the expectation conditioned on the preceding dataset to give initial intuition about why the scaling law may exist, as the expectation over the preceding dataset is technically hard to analyze. Even if the leading term on the right-hand side does not depend on $\bX_\mcD$, the $o_P(k^{-2})$ term may have a non-negligible expectation. We leave a more complete characterization as a future direction.






Next, we consider the case of M-estimators, a more general family of
models that minimize an empirical loss function, 
including robust least squares and GLMs like logistic regression.
Our model in this case has parameters $\theta \in \R^d$, and we consider a per-example loss $\ell(\theta; x, y)$ with a metric $\mcL(\theta)$ on the model performance (typically an expectation $\mcL(\theta)$ over a population distribution $\mcQ$). We assume that the preceding dataset $\mcD = \{(x_i, y_i)\}_{i=1}^k$ is sampled i.i.d. from
$\mcQ$, and we
consider
the following M-estimation problem:
\begin{equation}
    \label{eqn:m-estimator-loo}
    \hat\theta_{\mcD} = \argmin_{\theta\in\R^d} \; \sum_{i=1}^{k}\ell(\theta; x_i,y_i).
\end{equation}
Given a new training example $z = (x, y)$, we also define the corresponding problem with one additional loss term:
\begin{equation}
    \label{eq:m-estimator}
    \hat\theta_{\mcD\cup \{z\}} = \argmin_{\theta\in\R^d} \; \left(\ell(\theta; x,y)+\sum_{i=1}^{k}\ell(\theta; x_i,y_i)\right).
\end{equation}
Our goal is to evaluate the difference in model performance from including the single data point $z$, and the marginal contribution is $\Delta(z,\mcD) = \mcL(\hat\theta_{\mcD})-\mcL(\hat\theta_{\mcD\cup \{z\}})$. To characterize this term we define the population minimizer as
$$
    \theta^* = \argmin_{\theta \in \R^d} \; \E_{(x',y')\sim \mcQ}[\ell(\theta; x',y')],
$$
and we denote the population Hessian matrix at $\theta^*$ as 
$$
V_{\theta^*} = \nabla^2\E_{(x',y')\sim \mcQ}[\ell(\theta^*;x',y')].
$$
Based on these terms, we can characterize the marginal contribution as follows. 




\begin{theorem} (Informal)
    \label{thm:m-estimator}
    Under regularity conditions on $\ell, \mcL$ and $\mcQ$,
    \begin{enumerate}[(1)]
        \item if $\nabla\mcL(\theta^*) \neq 0$, the marginal contribution satisfies
        $$
            \Delta(z,\mcD) = \frac{1}{k}\nabla\mcL(\theta^*)^\top V_{\theta^*}^{-1}\nabla\ell(\theta^*;x,y)+O_P\left(\frac{1}{k^{3/2}}\right),
        $$
        \item if $\nabla\mcL(\theta^*) = 0$, then $\sqrt{k}\nabla\mcL(\hat\theta_\mcD) = O_P(1)$ and
        \begin{align*}
            \Delta(z,\mcD) =& \frac{1}{k^{3/2}}(\sqrt{k}\nabla\mcL(\hat\theta_\mcD))^\top V_{\theta^*}^{-1}\nabla\ell(\theta^*;x,y)+O_P\left(\frac{1}{k^{2}}\right).
        \end{align*}
    \end{enumerate}
\end{theorem}
A formal statement of the regularity conditions and the proof are deferred to \Cref{appendix: m estimator}.
    In principle, the model performance should be the loss function evaluated on the population level, i.e., $\mcL(\theta) = \E_{(x', y')\sim\mcQ}[\ell(\theta; x', y')]$. In this case, the optimality of $\theta^*$ implies $\nabla \mcL(\theta^*)=0$. However, the underlying distribution $\mcQ$ is often intractable, so a common surrogate is the average loss on a held-out validation dataset $\{(x'_j, y'_j)\}_{j=1}^m$, i.e., $\hat{\mcL}(\theta) = \frac{1}{m}\sum_{j=1}^{m}\ell(\theta; x'_j, y'_j)$. The minimizer of this finite-sample metric is generally different from the population minimizer $\theta^*$, and we therefore have $\nabla \hat{\mcL}(\theta^*) \neq 0$. According to \Cref{thm:m-estimator}, the asymptotic rate can differ under these two settings. In practice, the test set typically has a comparable size to the training set, hence we should expect the empirical behavior to be mixed. We also remark that
    even if $\sqrt{k}\nabla\mcL(\hat\theta_\mcD) = O_P(1)$, its expectation could be non-negligible and the leading term could be $O_P(k^{-2})$ in the second case. This occurs in the linear regression model according to \Cref{thm:linear}, but for general M-estimators, it is not clear what the order of this term is after taking the expectation across $\mcD$.

Overall, our theoretical analysis provides support for a scaling exponent $\alpha(z)$ in the marginal contribution with respect to the dataset size, and some examples of factors that affect the coefficient $c(z)$. However, we
emphasize that our theoretical analysis is limited, as it requires strong regularity assumptions on both the models and data distributions. Moreover, our results are limited to single marginal contributions $\Delta(z,\mcD)$, and the expectation of the $O_P$ terms may not be negligible even in the limit. Therefore, these results do not make precise predictions about the exact
parameters for the scaling law in practice, which also involves finite sample bias from both training and evaluation. These factors explain why our empirical results show heterogeneity in the scaling parameter $\alpha(z)$, although it is often concentrated in a range from $[1, 1.5]$ that is consistent with the loose predictions from our theoretical results.
We remark that such disagreements between theory and practice are also widely observed with aggregate scaling laws \citep{hutter2021learning}.

\subsection{Proof of \Cref{thm:linear}}
\label{appendix: linear model}

\begin{proof}
In the following analysis we denote $\bX_{\mcD} = (X_1,\cdots, X_n)^T\in \R^{n\times p}, \by_{\mcD} = (y_1, \cdots, y_n)^T \in \R^{n}, \beps_{\mcD} = (\epsilon_1, \cdots, \epsilon_n)^T \in \R^{n} $ for convenience. 
    Notice that:
\begin{align*}
    \mcL(\beta) = \E_{(x',y')\sim \mcP}[(x'^T\beta^{\star} + \epsilon -x'^T\beta)^2] = \sigma^2 + (\beta-\beta^\star)^T\Sigma(\beta-\beta^\star),
\end{align*}
\begin{align*}
    \hat\beta_{\mcD} - \beta^\star = (\bX^T_{\mcD}\bX_{\mcD})^{-1}\bX_{\mcD}^T\beps_{\mcD}, \quad     \hat\beta_{\mcD\cup\{n+1\}} - \beta^\star = (\bX^T_{\mcD}\bX_{\mcD}+xx^T)^{-1}(\bX_{\mcD}^T\beps_{\mcD}+x\epsilon)
\end{align*}
Then, in expectation over $\beps_{\mcD}$ (i.e., expectation conditional on $\bX_\mcD$ and $z$), we have 
\begin{equation*}
\begin{aligned}
    \E[\Delta(z,\mcD)|z,\bX_\mcD] = &\E_{\beps_\mcD}[\beps_{\mcD}^T\bX_{\mcD}(\bX^T_{\mcD}\bX_{\mcD})^{-1}\Sigma(\bX^T_{\mcD}\bX_{\mcD})^{-1}\bX_{\mcD}^T\beps_{\mcD}]\\
    &-\E_{\beps_\mcD}[(\bX_{\mcD}^T\beps_{\mcD}+x\epsilon)^T(\bX^T_{\mcD}\bX_{\mcD}+xx^T)^{-1}\Sigma(\bX^T_{\mcD}\bX_{\mcD}+xx^T)^{-1}(\bX_{\mcD}^T\beps_{\mcD}+x\epsilon)]\\
    =&\sigma^2\tr(\bX_{\mcD}(\bX^T_{\mcD}\bX_{\mcD})^{-1}\Sigma(\bX^T_{\mcD}\bX_{\mcD})^{-1}\bX_{\mcD}^T)\\
    &- \epsilon^2x^T(\bX^T_{\mcD}\bX_{\mcD}+xx^T)^{-1}\Sigma(\bX^T_{\mcD}\bX_{\mcD}+xx^T)^{-1}x\\
    &-\sigma^2\tr(\bX_{\mcD}(\bX^T_{\mcD}\bX_{\mcD}+xx^T)^{-1}\Sigma(\bX^T_{\mcD}\bX_{\mcD}+xx^T)^{-1}\bX_{\mcD}^T)\\
    =&\sigma^2\tr((\bX^T_{\mcD}\bX_{\mcD})^{-1}\Sigma - (\bX^T_{\mcD}\bX_{\mcD}+xx^T)^{-1}\Sigma)\\
    &- (\epsilon^2-\sigma^2)x^T(\bX^T_{\mcD}\bX_{\mcD}+xx^T)^{-1}\Sigma(\bX^T_{\mcD}\bX_{\mcD}+xx^T)^{-1}x
\end{aligned}
\end{equation*}
Applying the Sherman–Morrison formula, we know that
\begin{equation*}
    (\bX^T_{\mcD}\bX_{\mcD}+xx^T)^{-1} = (\bX^T_{\mcD}\bX_{\mcD})^{-1} - \frac{(\bX^T_{\mcD}\bX_{\mcD})^{-1}xx^T (\bX^T_{\mcD}\bX_{\mcD})^{-1}}{1+x^T(\bX^T_{\mcD}\bX_{\mcD})^{-1}x}.
\end{equation*}
By the law of large numbers we have $\frac{1}{k}\bX^T_{\mcD}\bX_{\mcD}\xrightarrow{p}\Sigma$, and therefore by the continuous mapping theorem we have $k(\bX^T_{\mcD}\bX_{\mcD})^{-1}\xrightarrow{p}\Sigma^{-1}$ and $k(\bX^T_{\mcD}\bX_{\mcD}+xx^T)^{-1}\xrightarrow{p}\Sigma^{-1}$. 
Hence,
\begin{equation*}
    \begin{aligned}
         \E[\Delta(z,\mcD)|z,\bX_\mcD] =& \frac{\sigma^2x^T(\bX^T_{\mcD}\bX_{\mcD})^{-1}\Sigma(\bX^T_{\mcD}\bX_{\mcD})^{-1}x}{1+x^T(\bX^T_{\mcD}\bX_{\mcD})^{-1}x}-(\epsilon^2-\sigma^2)x^T(\bX^T_{\mcD}\bX_{\mcD}+xx^T)^{-1}\Sigma(\bX^T_{\mcD}\bX_{\mcD}+xx^T)^{-1}x\\
         =&\frac{1}{k^2}\frac{\sigma^2x^T\Sigma^{-1}\Sigma\Sigma^{-1}x+o_P(1)}{1+o_P(1)}-\frac{1}{k^2}\left((\epsilon^2-\sigma^2)x^T\Sigma^{-1}\Sigma\Sigma^{-1}x+o_P(1)\right)\\
         =&\frac{2\sigma^2-\epsilon^2}{k^2}x^T\Sigma^{-1} x+o_P\left(\frac{1}{k^2}\right),
    \end{aligned}
\end{equation*}
which finishes the proof.
\end{proof}

\subsection{Proof of \Cref{thm:m-estimator}}
\label{appendix: m estimator}
In this section, we prove the individualized scaling law for M-estimators. For convenience, we index the point to be evaluated as $(x_{k+1}, y_{k+1})$, and write $\hat\theta_{k+1} = \hat\theta_{\mcD\cup\{z\}}, \hat\theta_{k} = \hat\theta_{\mcD}$ in the following analysis. For this general model family, we cannot obtain a closed-form solution as before, and we will focus on the asymptotics instead. A standard result on the asymptotic distribution of M-estimators is the following. 

\begin{lemma}[Asymptotic normality of M-estimators] 
\label{lem: M estimator asymptotics}
    Assume that the loss function $\ell(\theta; x, y)$ is twice differentiable in $\theta$ such that, for every $\theta_1,\theta_2$ in a neighborhood of $\theta^*$ and a measurable function $M(x, y)$ with $\E_{\mcQ}[M(x, y)]\leq\infty$, 
    $$
    \|\ell(\theta_1; x, y)-\ell(\theta_2; x, y)\|\leq M(x, y)\|\theta_1-\theta_2\|.
    $$
    $$
    \|\nabla\ell(\theta_1; x, y)-\nabla\ell(\theta_2; x, y)\|\leq M(x, y)\|\theta_1-\theta_2\|.
    $$
    Also assume that $\E_\mcQ[\|\nabla\ell(\theta^*; x, y)\|_2^2]\leq\infty$, $\theta\rightarrow\E_\mcQ\nabla\ell(\theta; x, y)$ is differentiable at $\theta^*$ with nonsingular derivative matrix $V_{\theta^{*}}$. If $\hat\theta_k\xrightarrow[]{P}\theta^*$ and $k^{-1}\sum_{i=1}^k\nabla\ell(\hat\theta_k; x_i, y_i)=o_p(k^{-1/2})$, then $\sqrt{k}(\hat\theta_{\mcD}-\theta^*)$ is asymptotically normal with mean zero and covariance matrix $V_{\theta^{*}}^{-1}\E_{\mcQ}[\nabla\ell(\theta^*; x, y)\nabla\ell(\theta^*; x, y)^T]V_{\theta^{*}}^{-1}$.
\end{lemma}
The proof can be found in \citet{van2000asymptotic} (see Theorem 5.21 therein). Define the empirical Hessian matrix evaluated on the empirical minimizers as
$$
    \hat H_{k+1} := \frac{1}{k+1} \sum_{i=1}^{k+1} \nabla^2 \ell\left(\hat{\theta}_{k+1}; x_i, y_i\right).
$$
Also, define for $1 \leq i \leq k+1$,
$$
\delta_{i, k+1}:=\frac{(k+1)^{-1}\left\|{\hat H}_{k+1}^{-1} \nabla\ell\left(\hat{\theta}_{k+1}; x_i, y_i\right)\right\|_2}{1-(k+1)^{-1}\left\|{\hat H}_{k+1}^{-1} \nabla^2\ell\left(\hat{\theta}_{k+1}; x_i, y_i\right)\right\|_{op}}.
$$
The main technique we use to prove \Cref{thm:m-estimator} is the following inequality on leave-one-out M-estimators from \citet{kuchibhotla2018deterministic}.
\begin{lemma}[Corollary 7.1 from \citealt{kuchibhotla2018deterministic}]
\label{lem:smooth M estimator}
Assume that $\ell(\theta; x,y)$ is convex and twice differentiable in $\theta$ for every $(x, y)$. If $\delta_{i, k+1} \geq 0$ for all $1 \leq i \leq k+1$ and
$$
\max _{1 \leq i \neq j \leq n} C\left(\frac{3}{2} \delta_{i, k+1}; x_j, y_j\right) \leq \frac{4}{3}
$$
then 
$$
\begin{aligned}
&\| \hat{\theta}_{k}  -\hat{\theta}_{k+1}-(k+1)^{-1} \hat {H}_{k+1}^{-1} \nabla \ell(\hat{\theta}_{k+1}; x_{k+1}, y_{k+1}) \|_2 \\
\leq& \frac{3 \delta_{k+1, k+1}}{2}\left[\max _{1 \leq i \leq k} C\left(\frac{3}{2} \delta_{k+1, k+1}; x_i, y_i\right)-1+(k+1)^{-1}\left\|\hat {H}_{k+1}^{-1} \nabla^2 \ell(\hat{\theta}_{k+1}; x_{k+1}, y_{k+1})\right\|_{op}\right] .
\end{aligned}
$$    
\end{lemma}
A few more assumptions we need to analyze the marginal contribution are:
\begin{assumption}
    \label{asm:metric}
     The metric $\mcL(\theta)$ is twice differentiable, and $\|\nabla\mcL(\theta)\|_2, \|\nabla^2\mcL(\theta)\|_{op}$ are bounded by a universal constant $C_1>0$.
\end{assumption}
\begin{assumption}
    \label{asm:smooth}
    For $u\geq 0$, the function 
    \begin{equation}
        C(u; x, y) = \sup_{\|\theta_1-\theta_2\|_2\leq u}\sup_{\|e\|_2 = 1} \frac{e^T\nabla^2\ell(\theta_1; x, y)e}{e^T\nabla^2\ell(\theta_2; x, y)e}
    \end{equation}
    is differentiable with its derivative bounded by a universal constant $C_2$ at a neighborhood of $u=0$ for any $(x,y)$.
\end{assumption}
Now we are ready to state and prove the formal version of \Cref{thm:m-estimator}.
\begin{theorem}[Formal version of \Cref{thm:m-estimator}]
\label{thm: m estimator appendix}
    Under the conditions of \Cref{lem: M estimator asymptotics} and \Cref{lem:smooth M estimator}, assume that \Cref{asm:metric}, \Cref{asm:smooth} hold, and $\E_{\mcQ}[\|\nabla^2\ell(\theta;x,y)\|]\leq\infty$, then:
    \begin{enumerate}[(1)]
        \item if $\nabla\mcL(\theta^*) \neq 0$, the marginal contribution satisfies
$$
\Delta(z,\mcD) = \frac{1}{k}\nabla\mcL(\hat\theta^*)^TV_{\theta^*}^{-1}\nabla\ell(\theta^*;x,y)+O_P(\frac{1}{k^{3/2}}),
$$
        \item if $\nabla\mcL(\theta^*) = 0$, then $\sqrt{k}\nabla\mcL(\hat\theta_k) = O_P(1)$ and
$$
\Delta(z,\mcD) = \frac{1}{k^{3/2}}(\sqrt{k}\nabla\mcL(\hat\theta_k))^TV_{\theta^*}^{-1}\nabla\ell(\theta^*;x,y)+O_P(\frac{1}{k^{2}}),
$$
\end{enumerate}
\end{theorem}
\begin{proof}[Proof of \Cref{thm:m-estimator}]
Applying Taylor's theorem to the marginal contribution, we have 
\begin{equation}
\begin{aligned}
\label{eqn: Taylor}
    \Delta(z,\mcD) =& \mcL(\hat\theta_{k}) - \mcL(\hat\theta_{k+1})=\nabla\mcL(\theta_{k})^T(\hat\theta_{k}-\hat\theta_{k+1})+\frac{1}{2}(\hat\theta_{k}-\hat\theta_{k+1})^T\nabla^2\mcL(\xi)(\hat\theta_{k}-\hat\theta_{k+1}),
\end{aligned}
\end{equation}
where $\xi = \alpha\hat\theta_{k}+(1-\alpha)\hat\theta_{k+1}$ for some $\alpha\in [0,1]$. Then notice that 
$$
\|\hat{H}_{k+1} -\frac{1}{k+1} \sum_{i=1}^{k+1} \nabla^2 \ell\left(\theta^*; x_i, y_i\right)\|_{op}\leq\frac{1}{k+1} \sum_{i=1}^{k+1}M(x_i,y_i)\|(\hat\theta_{k+1}-\theta^*)\|_2,
$$
from \Cref{lem: M estimator asymptotics}, we know that $\hat\theta_{k+1}-\theta^*_2=O_P(k^{-1/2})$, and by law of large number, $\frac{1}{k+1} \sum_{i=1}^{k+1}M(x_i,y_i)=O_P(1)$, hence $\hat{H}_{k+1} -\frac{1}{k+1} \sum_{i=1}^{k+1} \nabla^2 \ell\left(\theta^*; x_i, y_i\right)=O_P(k^{-1/2})$. On the other hand, from the central limit theorem, 
$$
\frac{1}{k+1} \sum_{i=1}^{k+1} \nabla^2 \ell\left(\theta^*; x_i, y_i\right) - \E\nabla^2 \ell\left(\theta^*; x_i, y_i\right)=O_P(k^{-1/2}).
$$
Use the fact that $\E\nabla^2 \ell\left(\theta^*; x_i, y_i\right) = \nabla\E\nabla \ell\left(\theta^*; x_i, y_i\right)=V_{\theta^*}$, we conclude that $\hat{H}_{k+1}-V_{\theta^*} = O_P(k^{-1/2})$ and hence 
\begin{align}
    \label{eqn: inverse H}
    \hat{H}_{k+1}^{-1}-V_{\theta^*}^{-1} = O_P(k^{-1/2})
\end{align}
with the fact that $V_{\theta^*}$ is nonsingular.

Moreover, since $\hat\theta_{k+1}-\theta^*=O_P(k^{-1/2})$, and $\nabla\ell, \nabla^2\ell$ are continuous and Lipscitz, we know that 
\begin{align}
    \label{eqn: grad 1 diff}
    \nabla \ell(\hat{\theta}_{k+1}; x_{k+1}, y_{k+1})-\nabla \ell(\theta^*; x_{k+1}, y_{k+1})=O_P(k^{-1/2}),
\end{align}
and
\begin{align}
    \label{eqn: grad 2 diff}
     \nabla^2 \ell(\hat{\theta}_{k+1}; x_{k+1}, y_{k+1})-\nabla^2 \ell(\theta^*; x_{k+1}, y_{k+1})=O_P(k^{-1/2}).
\end{align} 
Based on \eqref{eqn: inverse H}-\eqref{eqn: grad 2 diff}, we shall conclude that 
$
\delta_{k+1,k+1} = O_P(k^{-1}),
$
and then use \Cref{asm:smooth}, 
\begin{align}
\label{eqn: C function}
    \max_{1\leq i\leq k}C\left(\frac{3}{2}\delta_{k+1,k+1};x_i,y_i\right)-1=O_P(k^{-1}).
\end{align}
Combining \eqref{eqn: inverse H}-\eqref{eqn: C function} with \Cref{lem:smooth M estimator}, we can conclude that

$$
\hat\theta_{k}-\hat\theta_{k+1} = \frac{1}{k+1}V_{\theta^*}^{-1}\nabla\ell(\theta^*;x_{k+1},y_{k+1})+O_P(k^{-3/2}).
$$
Then plug it into \eqref{eqn: Taylor}, we know that 
\begin{equation}
\label{eqn: Delta form}
\Delta(z,\mcD) = \frac{1}{k+1}\nabla\mcL(\hat\theta_k)^T(V_{\theta^*}^{-1}\nabla\ell(\theta^*;x_{k+1},y_{k+1})+O_P(k^{-3/2}))+O_P(k^{-2}).    
\end{equation}

Applying the Delta method to $\nabla\mcL(\hat\theta_k)$, the limiting distribution of $\nabla\mcL(\hat\theta_k)$ is 
\begin{equation}
\label{eqn: limit mcL}
    \sqrt{k}(\nabla\mcL(\hat\theta_k)-\nabla\mcL(\theta^*))\xrightarrow[]{d}\mcN\left(0, \nabla^2\mcL(\theta^*)V_{\theta^{*}}^{-1}\E_{\mcQ}[\nabla\ell(\theta^*; x, y)\nabla\ell(\theta^*; x, y)^T]V_{\theta^{*}}^{-1}\nabla^2\mcL(\theta^*)\right).
\end{equation}
If $\nabla\mcL(\theta^*) \neq 0$, then \eqref{eqn: Delta form} implies
$$
\Delta(z,\mcD) = k^{-1}\nabla\mcL(\hat\theta^*)^TV_{\theta^*}^{-1}\nabla\ell(\theta^*;x_{k+1},y_{k+1})+O_P(k^{-3/2}).
$$
And if $\nabla\mcL(\theta^*) = 0$, then $\sqrt{k}\nabla\mcL(\hat\theta_k) = O_P(1)$ and \eqref{eqn: Delta form} implies
$$
\Delta(z,\mcD) = k^{-3/2}(\sqrt{k}\nabla\mcL(\hat\theta_k))^TV_{\theta^*}^{-1}\nabla\ell(\theta^*;x_{k+1},y_{k+1})+O_P(k^{-2}),
$$
which finishes the proof.
\end{proof}

\begin{remark}
Although we introduce many assumptions to establish \Cref{thm: m estimator appendix}, many of them are natural and could be verified. For example, for a logistic regression model, the loss function is $\ell(\theta; x, y) = -\theta^Txy+\log(1+\exp(\theta^Tx))$, so the first and second order derivatives are
\begin{equation}
    \nabla \ell(\theta; x, y) = -x\left(y-\frac{\exp(\theta^Tx)}{1+\exp(\theta^Tx)}\right), \quad \nabla^2 \ell(\theta; x, y)= \frac{\exp(\theta^Tx)}{(1+\exp(\theta^Tx))^2}xx^T.
\end{equation}
Moreover, the third derivative can be written as 
\begin{equation}
    \nabla^3 \ell(\theta; x, y) = \frac{\exp(\theta^Tx)(1-\exp(\theta^Tx))} {(1+\exp(\theta^Tx))^3}x\otimes x\otimes x.
\end{equation}
Then the Lipschitz condition in \Cref{lem: M estimator asymptotics} can be guaranteed if the $\E_\mcQ[\|x\|_2^3]<\infty$. Suppose the metric is either $\mcL_1(\theta) = \E_{(x', y')\sim\mcQ}[\ell(\theta; x', y')]$ or $\mcL_2(\theta) = \frac{1}{m}\sum_{j=1}^{m}\ell(\theta; x'_j, y'_j)$, \Cref{asm:metric} and the requirement $\E_{\mcQ}[\|\nabla^2\ell(\theta;x,y)\|]\leq\infty$ is also implied by $\E_\mcQ[\|x\|_2^3]<\infty$. For the requirement on function $C$, the upper bound in \Cref{lem:smooth M estimator} could be verified by \eqref{eqn: C function}. In particular, since $\nabla^2\ell(\theta; x, y)$ is a rank-one matrix and depends on $\theta$ only through the scale, we have 
\begin{equation*}
\begin{aligned}
     C(u;x,y) =& \sup_{\|\theta_1-\theta_2\|_2\leq u} \frac{\exp(\theta_1^Tx)}{(1+\exp(\theta_1^Tx))^2}\frac{(1+\exp(\theta_2^Tx))^2}{\exp(\theta_2^Tx)}\\
     \leq&\sup_{\|\theta_1-\theta_2\|_2\leq u}\frac{\exp(\theta_1^Tx)}{\exp(\theta_2^Tx)}\sup_{\|\theta_1-\theta_2\|_2\leq u}\frac{\exp(\theta_2^Tx)^2}{\exp(\theta_1^Tx)^2}\\
     \leq&\exp(3u\|x\|_2).
\end{aligned}
\end{equation*}
Hence \Cref{asm:smooth} can be fulfilled by assuming $\|x\|_2$ is bounded. We want to highlight that many assumptions can be relaxed, and here we only state a basic version for clarity.
\end{remark}

\clearpage
\section{Likelihood-based scaling law estimator} \label{app:parametric-estimator}

Our objective when fitting the likelihood-based estimator is the following loss function,

\begin{equation*}
    \min_{c, \alpha, \sigma, \beta} \; \mcL(c, \alpha, \sigma, \beta) = \sum_{i = 1}^m \mathrm{NLL}(\Delta(z, \mcD_i), |\mcD_i|; c, \alpha, \sigma, \beta),
\end{equation*}

where the negative log-likelihood for each marginal contribution is based on Gaussian distributions with certain shared parameters:

\begin{equation*}
    \mathrm{NLL}(\Delta, k; c, \alpha, \sigma, \beta) = \frac{1}{2} \log (2\pi) + \frac{1}{2} \log (\sigma^2) - \frac{\beta}{2} \log (k) + \frac{\left( \Delta - c k^{-\alpha} \right)^2}{2 \sigma^2 k^{-\beta}}.
\end{equation*}

The objective is non-convex, and we find that directly optimizing over the four parameters leads to unstable results. We therefore use another approach that involves solving analytically for two of the four parameters. Given values for $\alpha$ and $\beta$, we can derive optimal values for $c$ and $\sigma$ in an analytic fashion. Beginning with $c$, we can see that its optimal value does not depend on $\sigma$, so we can write the optimal value $c(\alpha, \beta)$ as follows:


\begin{equation*}
    c(\alpha, \beta) 
    = \argmin_c \; \sum_{i = 1}^m k_i^{\beta} \left( \Delta_i - c k_i^{-\alpha} \right)^2
    = \frac{\sum_i k_i^{\beta - \alpha} \Delta_i}{\sum_i k_i^{\beta - 2 \alpha}}.
\end{equation*}

Next, we consider the optimal value for $\sigma^2$ given fixed values for $\alpha, \beta$ and the optimal value $c(\alpha, \beta)$. We solve this problem by reparameterizing in terms of the precision $\tau = 1 / \sigma^2$, and we use the solution to define the optimal value $\sigma(\alpha, \beta)$:


\begin{align*}
    \sigma(\alpha, \beta)
    = \left( \argmin_\tau \; \sum_{i = 1}^m - \log (\tau) + \tau \frac{\left( \Delta_i - c(\alpha, \beta) k_i^{-\alpha} \right)^2}{k^{-\beta}} \right)^{-\frac{1}{2}}
    = \sqrt{\frac{\sum_i k_i^\beta \left( \Delta_i - c(\alpha, \beta) k_i^{-\alpha} \right)^2}{m}}.
\end{align*}

Using these analytic expressions, we can define our objective solely over the exponents $\alpha$ and $\beta$. We therefore simplify the fitting process by optimizing the following loss:

\begin{equation*}
    \mcL(\alpha, \beta) = \mcL\left(c(\alpha, \beta), \alpha, \sigma(\alpha, \beta), \beta\right).
\end{equation*}

We optimize the two-parameter objective with Adam \citep{kingma2014adam}, and we find that this leads to stable results with a moderate number of gradient steps.

\textbf{Amortized estimator.} For the amortized estimator, we use several approaches to stabilize and improve the model's performance. Because we observe a wide range of values for the $c(z)$ coefficient from the non-amortized estimator, we set up the network to separately predict the sign and log-magnitude of this coefficient, or $\mathrm{sign}(c(z))$ and $\log |c(z)|$. We use gradient clipping, we clamp the predicted $\beta$ values to have a maximum value of $\exp(2)$ to avoid variance predictions too close to 0 at large $k$ values, and we apply a small prior to the predicted $\alpha(z)$ values via a penalty $(\alpha(z) - 1)^2$.

\clearpage
\section{Additional Experimental Results}
\label{appendix: experiments}

\subsection{Evidence for individualized data scaling laws}
\label{appendix: evidence}

Here, we provide experimental details for the results shown in \Cref{sec: evidence}. For SVMs we use the RBF kernel, and for neural networks we set the network width as 32, weight decay as $0.01$, and train using the Adam optimizer \cite{kingma2014adam} with learning rate $0.01$ for $200$ epochs. All of our models are run with scikit-learn package in Python. For each setting, we evaluated the model performance on a held-out test dataset of size $1000$. In addition, when generating the preceding dataset, we enforced the class balance to avoid degenerate scenarios where certain classes are not represented.

First, we show the estimated $\log|{c}(z)|$ and the relationship between $\alpha(z)$ and $\log| c(z)|$ in \Cref{fig: c}. The estimated $c(z)$ values exhibit a significant long-tail behavior in the original space, but we find that it has a well-shaped distribution in log-space. Moreover, we observe a strong positive correlation between $\alpha(z)$ and $\log |c(z)|$. This can be explained by the exponential effect of $\alpha(z)$ on $\psi_k(z)$: if $\alpha(z)$
increases by 1, $c(z)$ needs to increase by a factor of $k$
to keep $\psi_k(z)$ unchanged, therefore it is more reasonable to look at $\log |c(z)|$ values.
Similar to $\alpha (z)$, we also observe significant heterogeneity in the $\log|c(z)|$ values.
\begin{figure}[ht]
    \centering
    \subfigure[Histogram of estimated $\log|\hat c(z)|$]{\includegraphics[width=0.49\linewidth]{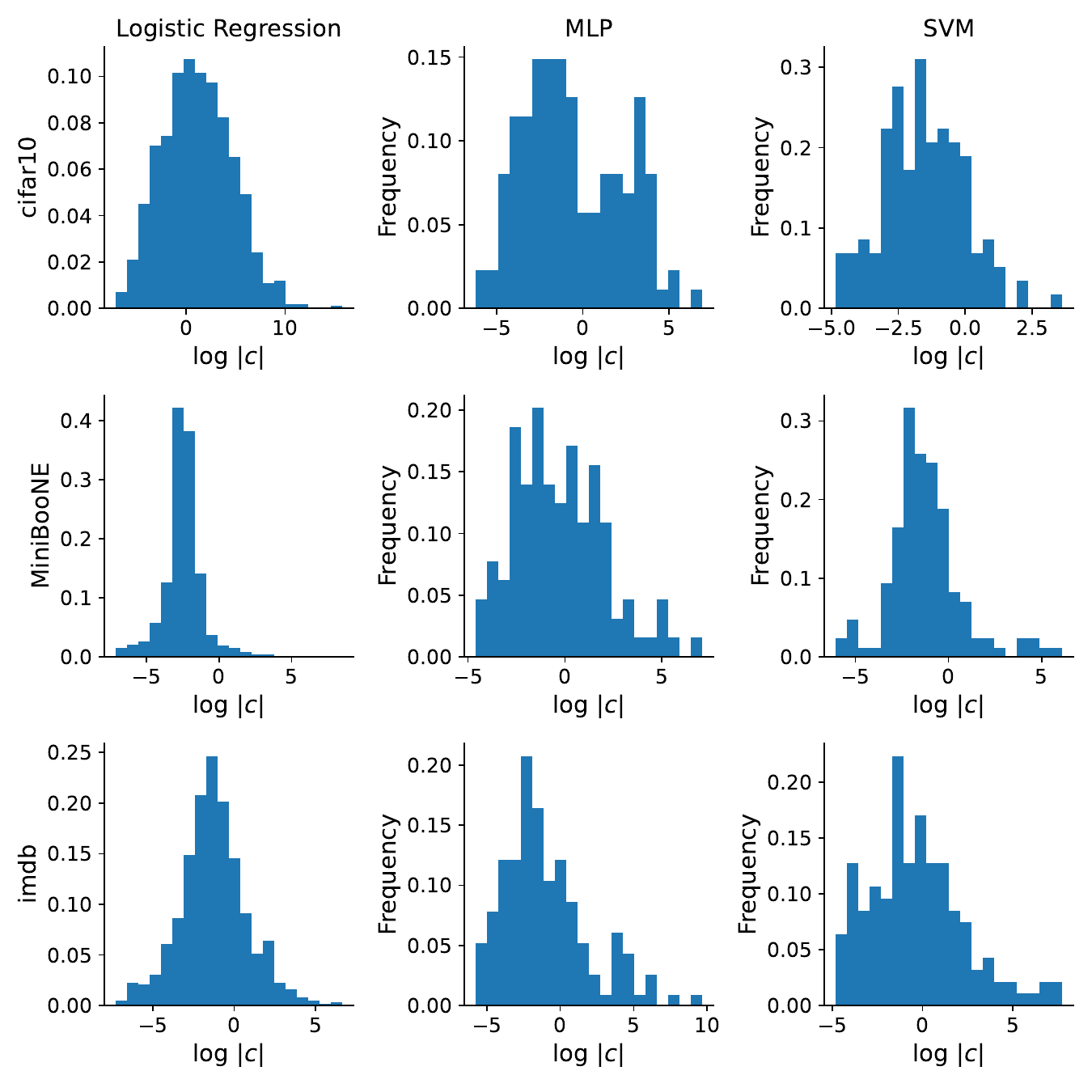}}
    \subfigure[Relationship between $\alpha$ and $\log| c(z)|$]{\includegraphics[width=0.49\linewidth]{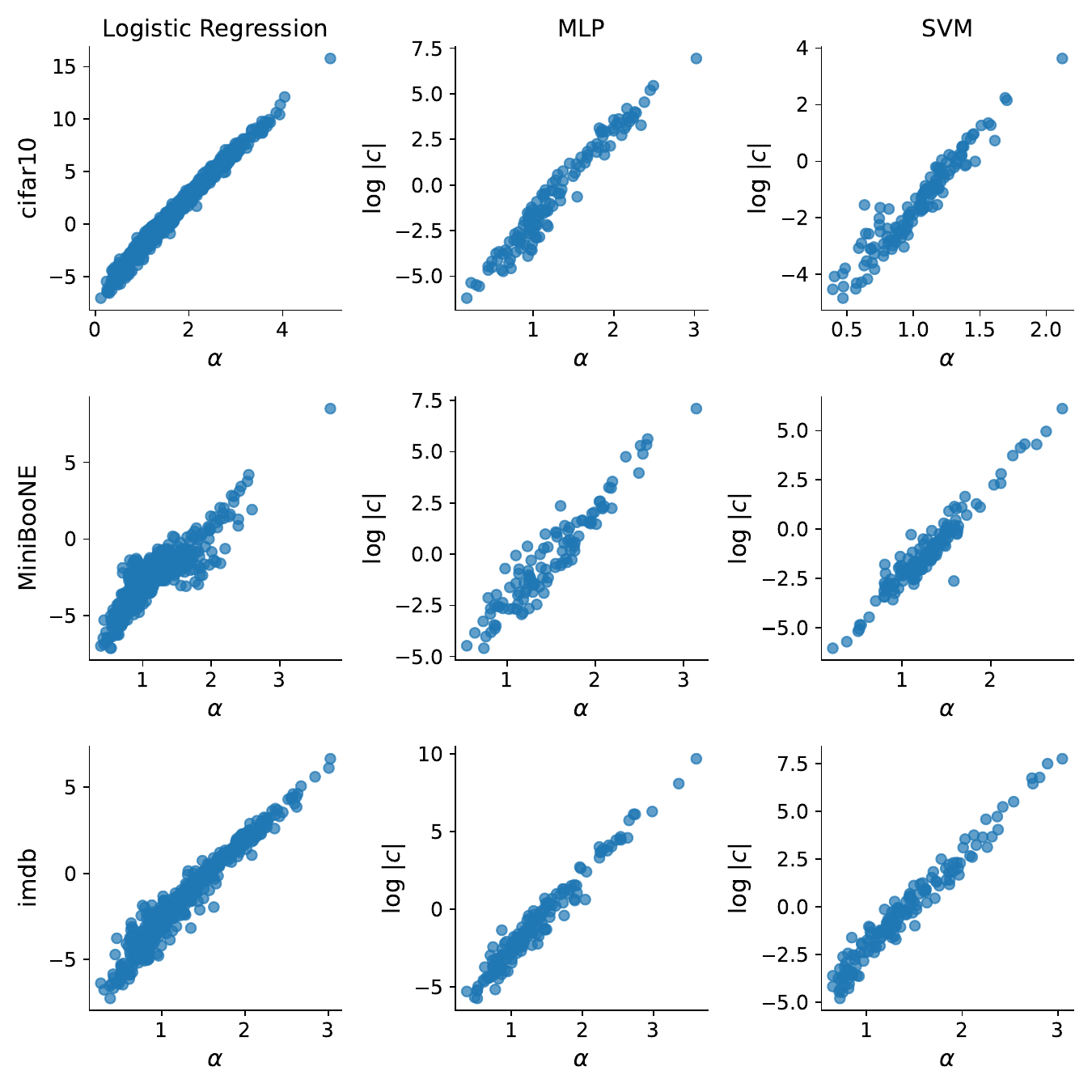}}
    \caption{\textbf{Histogram of estimated $\log|c(z)|$ and the relationship between $\alpha(z)$ and $\log| c(z)|$}: The estimated $c(z)$ values exhibit a well-shaped distribution in the log scale, and we observe a strong correlation between $\alpha(z)$ and $\log| c(z)|$. Similar to \Cref{fig: alpha}, we also exclude points with $R^2<0.8$ to make sure the estimates are reliable. }
    \label{fig: c}
\end{figure}

In \Cref{fig: r2}, we observe that most of the $R^2$ values are very close to 1, but there is still a small portion of points whose $R^2$ is below 1. We want to highlight that this is largely due to the finite sample error of estimating the mean at each cardinality. In \Cref{fig:shift}, we show the empirical distribution of $R^2$ using different numbers of samples, and we can observe that the distribution converges to 1 as the sample size increases; we generally expect all data points' $R^2$ values to be close to 1 when we have an infinite number of samples, which is consistent with our scaling law in \cref{eq:scaling-law}. Moreover, in \Cref{fig:r2_psi}, we examine the relationship between $R^2$ and $\log|\psi_k(z)|$. It is observed that for points with smaller $R^2$, they also tend to have smaller marginal contributions; this is most apparent when visualizing $\psi_k(z)$ for $k = 100$ and the average of $\psi_k(z)$ in the range from $100$ to $1000$ (\Cref{fig:r2_psi} left and bottom). This observation suggests that the
poorly-fitting points may not be due to the scaling law not applying, but to estimation noise that is relatively large when viewed in log-space.

\begin{figure}[ht]
    \centering
    \includegraphics[width=0.3\linewidth]{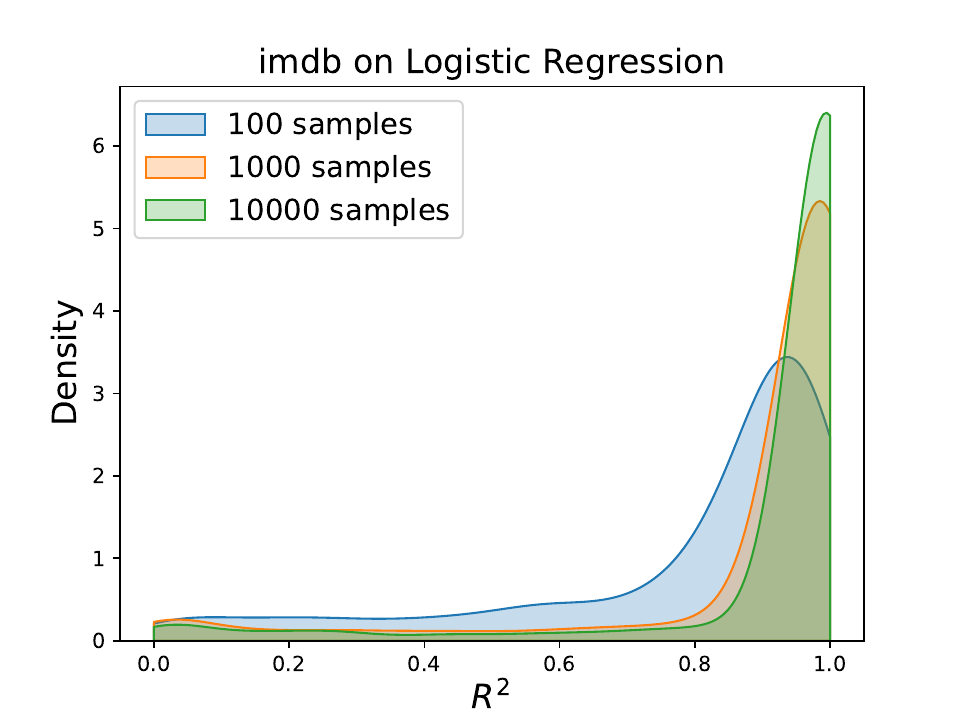}
        \includegraphics[width=0.3\linewidth]{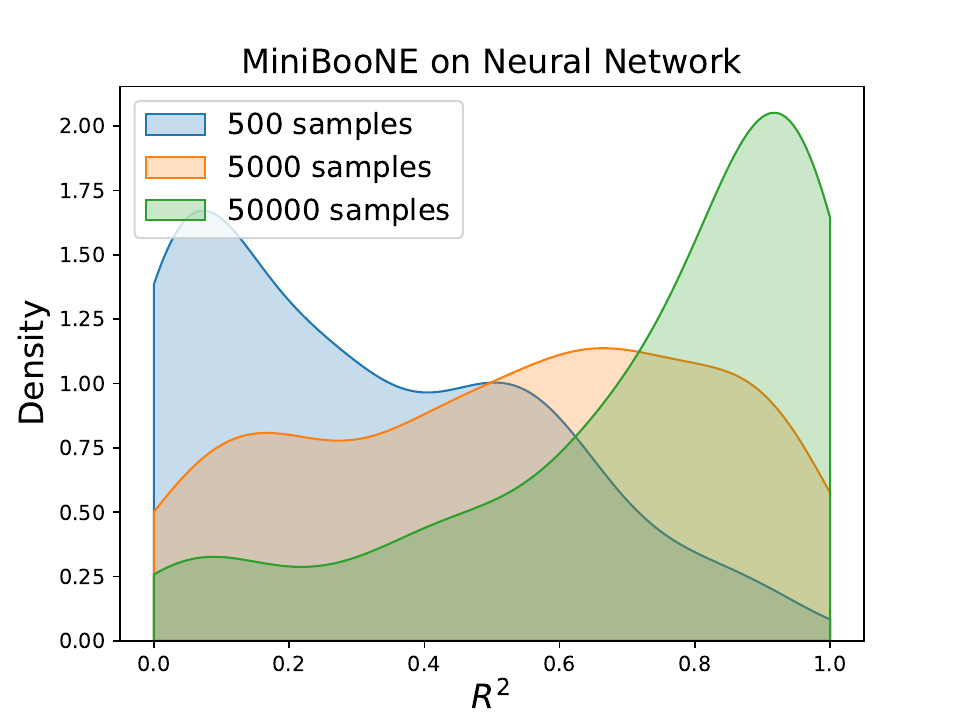}
    \includegraphics[width=0.3\linewidth]{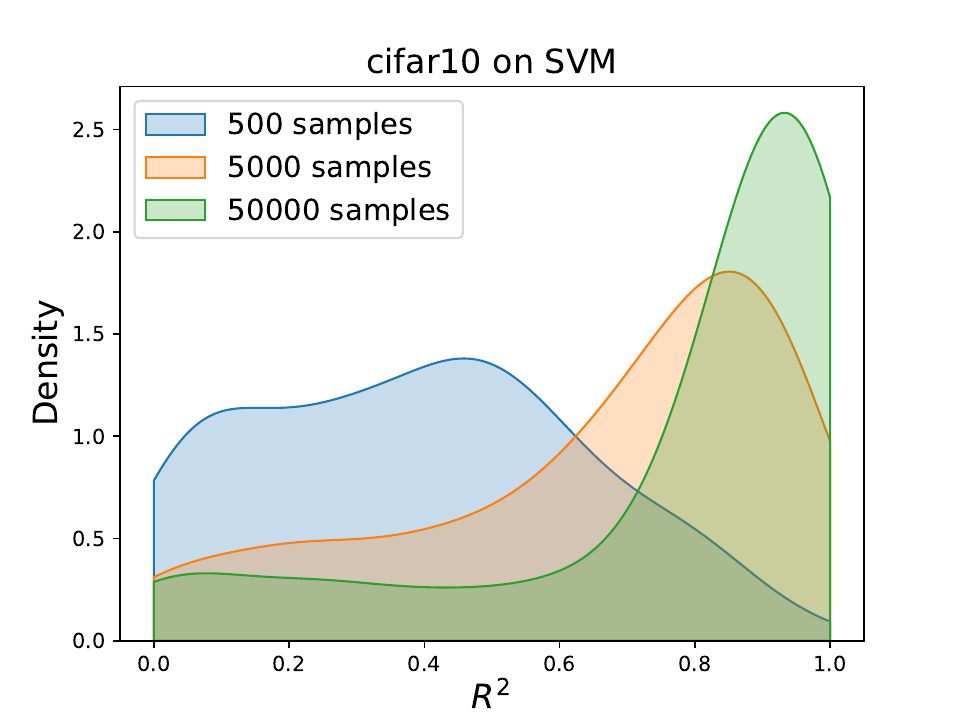}

    \caption{\textbf{Density of $R^2$ score with different numbers of samples.} As the sample size increases, the distribution of $R^2$ is concentrating towards 1.}
    \label{fig:shift}
\end{figure}
\begin{figure}
    \centering
    \subfigure[Relationship between $R^2$ and $\log|\psi_{100}(z)|$]{\includegraphics[width=0.49\linewidth]{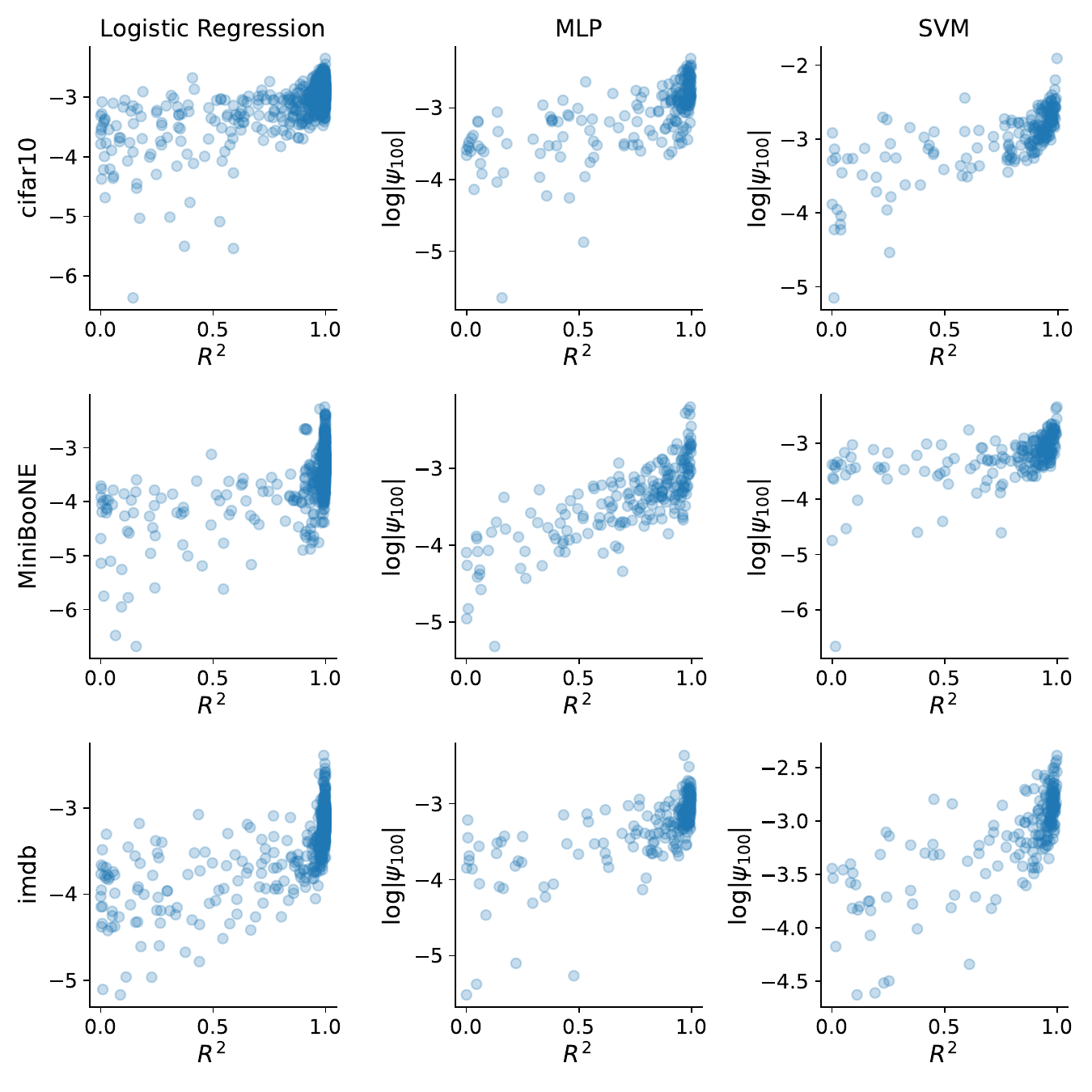}}
    \subfigure[Relationship between $R^2$ and $\log|\psi_{1000}(z)|$]{\includegraphics[width=0.49\linewidth]{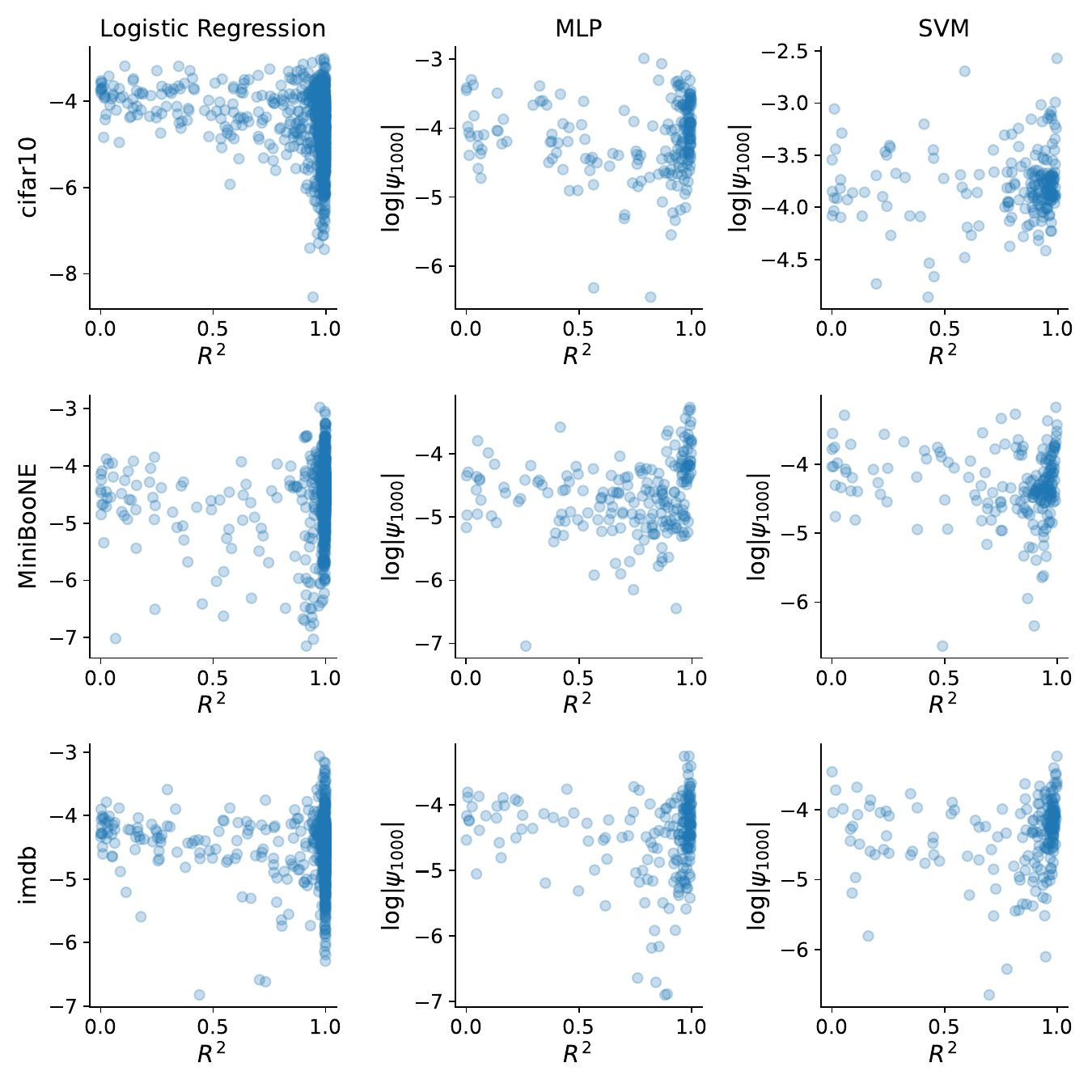}}
    \subfigure[Relationship between $R^2$ and $\log|\psi(z)|$, here $\psi(z)$ is defined as the average $\psi_k(z)$ over 10 log-spaced cardinalities between $k=100$ and $k=1000$.]{\includegraphics[width=0.49\linewidth]{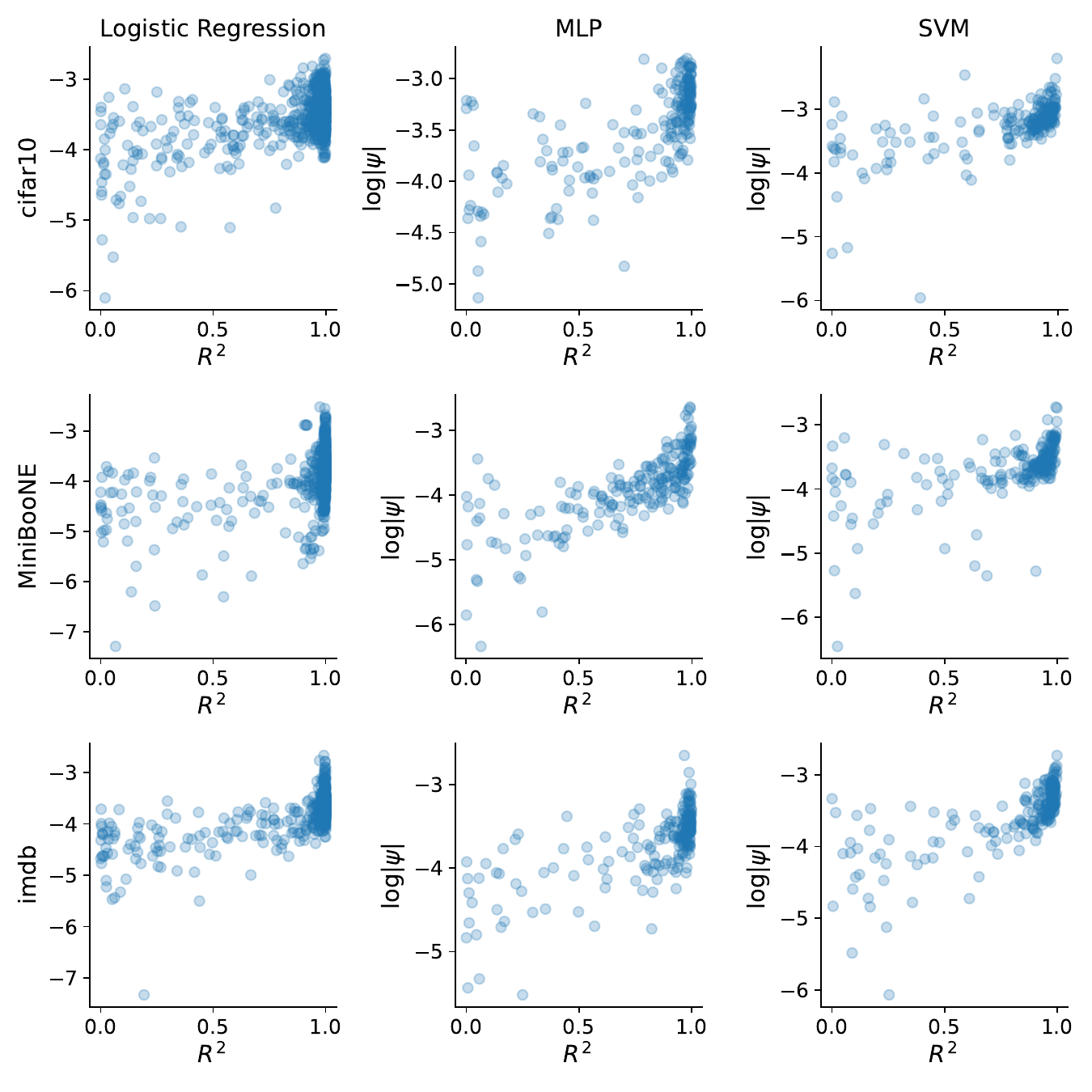}}
    \caption{\textbf{The relationship between $R^2$ score and the marginal contribution.} We examine the marginal contribution by evaluating $\log|\psi_k(z)|$ at $k=100$, $k=1000$, and an average over $10$ log-spaced cardinalities between $k=100$ and $k=1000$. For points with small $R^2$, they tend to have small marginal contributions. For $k=1000$, the marginal contributions of all points are dominated by the sampling error, therefore it is hard to distinguish points with different $R^2$.}
    \label{fig:r2_psi}
\end{figure}

Besides the evidence of the scaling law for the marginal contribution, we also provide evidence for the scaling law of the marginal contribution's variance, i.e., we aim to show that 
\begin{equation}
    \label{eqn:scaling var}
    \var(\Delta(z,\mcD)\mid |\mcD|=k) \approx \frac{\sigma^2(z)}{k^{\beta(z)}}.
\end{equation}

Similar to our approach in \Cref{sec:validation}, we can verify this scaling law by taking the following log transform,
\begin{equation}
\label{eqn:scaling var log}
    \log\var(\Delta(z,\mcD)\mid |\mcD|=k) \approx \log \sigma^2(z) - \beta(z)\log k,
\end{equation}
and evaluating the $R^2$ score of fitting a simple linear regression in log-space. The experiment results are shown in \Cref{fig:r2 var}, and we find that the $R^2$ scores for the variance are very close to 1 for all logistic regression and MLP models. For SVM, the computed $R^2$ are mostly close to 1, but there exists a portion of points that have a relatively small $R^2$. A possible reason is that the classifier returned by SVM depends only on the support vectors, hence whether or not the point is a support vector makes the marginal contribution differ significantly. Therefore, although the mean of those marginal contributions aligns well with the scaling law, the variance could be relatively unstable and result in a poor fit.

\begin{figure}
    \centering
    \includegraphics[width=0.5\linewidth]{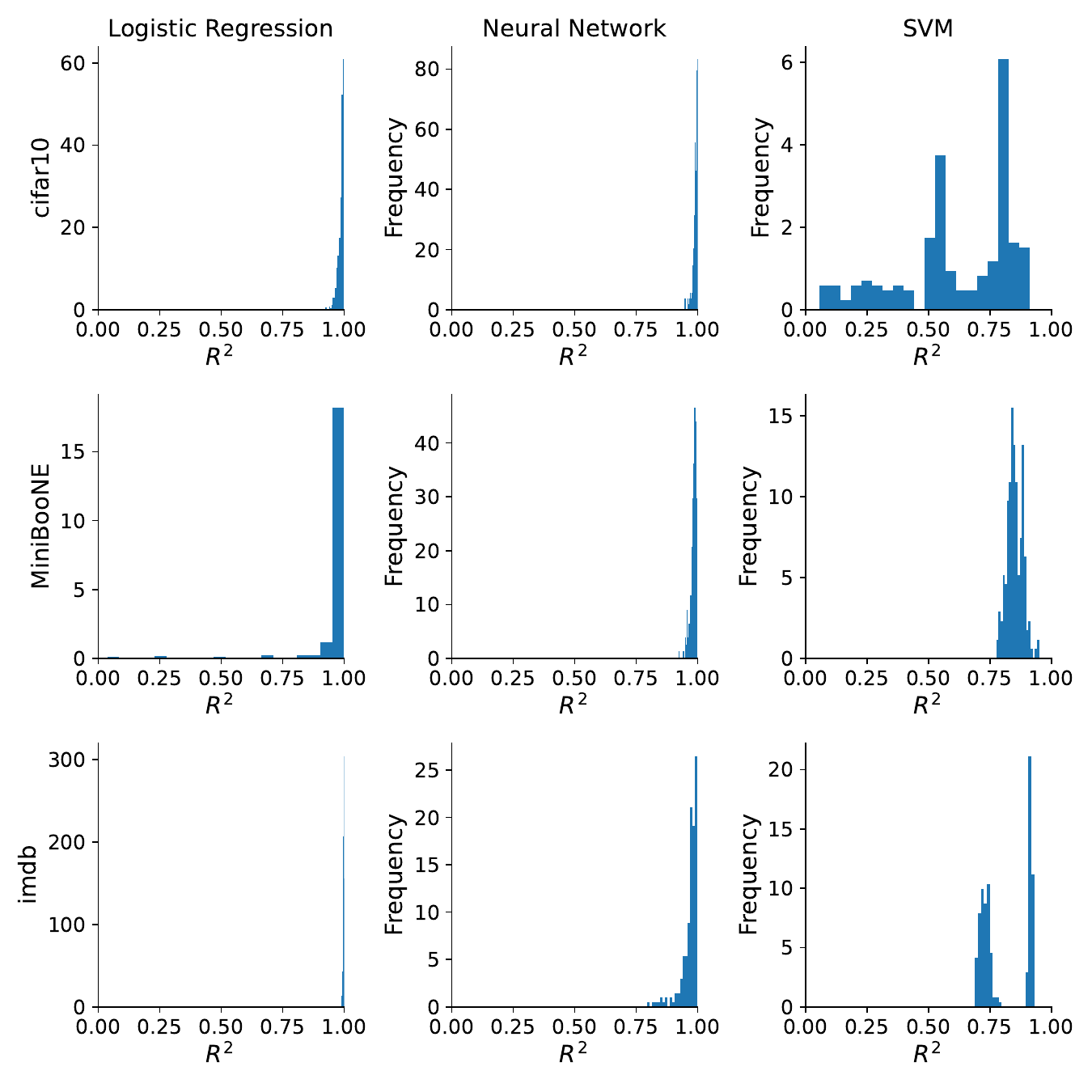}
    \caption{\textbf{Histogram of $R^2$ score of fitting the scaling law fitting on the variance:} for each setting, we compute the variance of sampled $\Delta(z, \mcD)$ with the same size $|\mcD|$, and then fit a simple linear regression according to \eqref{eqn:scaling var log} to obtain $R^2$}
    \label{fig:r2 var}
\end{figure}

Next, we examine the distribution of fitted parameters $\beta(z)$ and $\log \sigma^2(z)$. Similar to the $c(z)$ values, here we also observe a long-tail behavior of $\sigma^2(z)$ and hence plot $\log \sigma^2(z)$ instead. As we can observe, the variance of logistic regression models is decaying at a very fast rate, the majority of $\beta(z)$ values lie between $2$ and $4$. By contrast, for MLPs and SVMs, $\beta(z)$ is mostly between 1 and 2. Therefore, the signal-to-noise ratio of the marginal contribution is much lower for these models, making the estimation of the mean $\psi_k(z)$ at each cardinality more challenging. 

\begin{figure}
    \centering
    \subfigure[Histogram of $\hat\beta(z)$]{\includegraphics[width=0.49\linewidth]{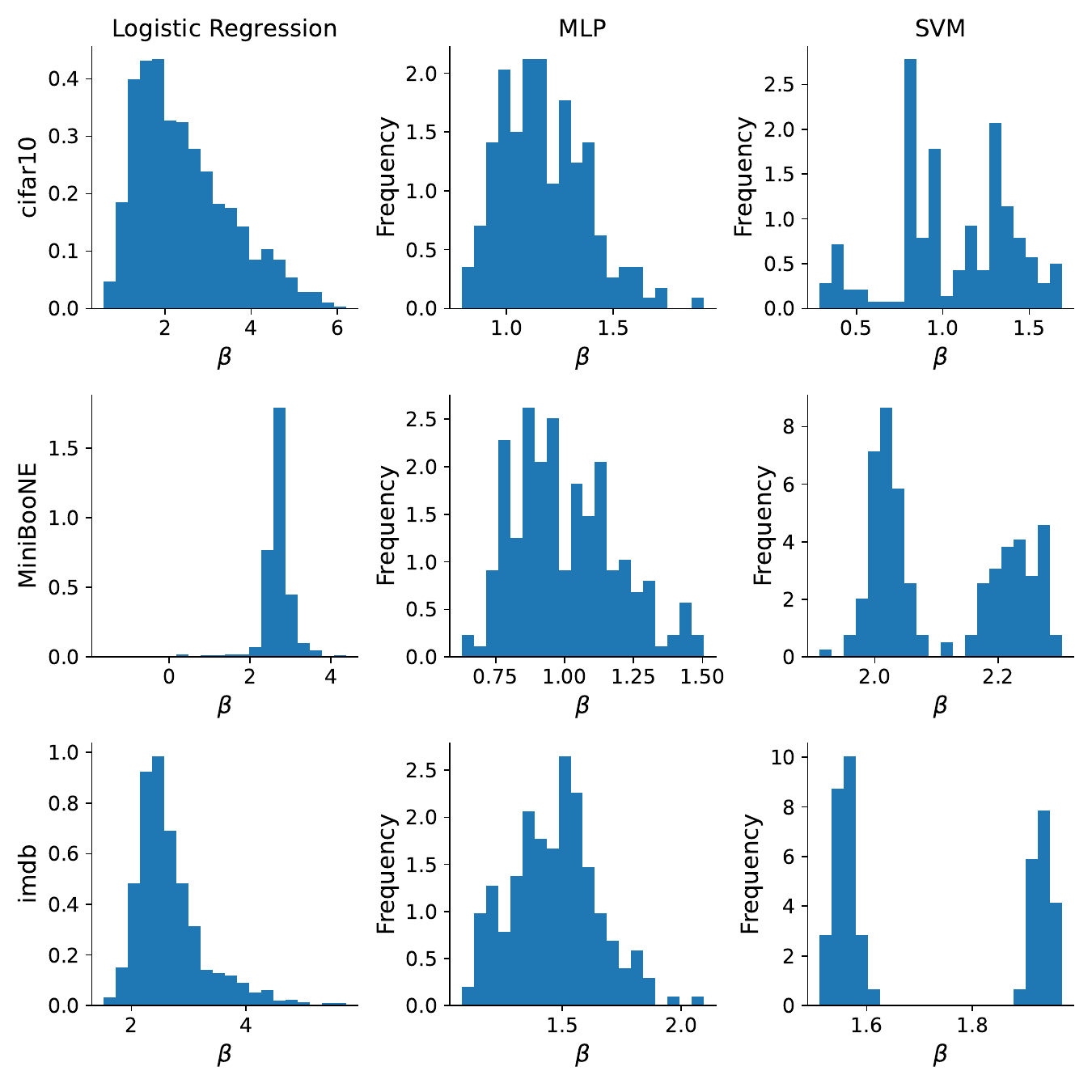}}
    \subfigure[Histogram of $\log \hat \sigma^2(z)$]{\includegraphics[width=0.49\linewidth]{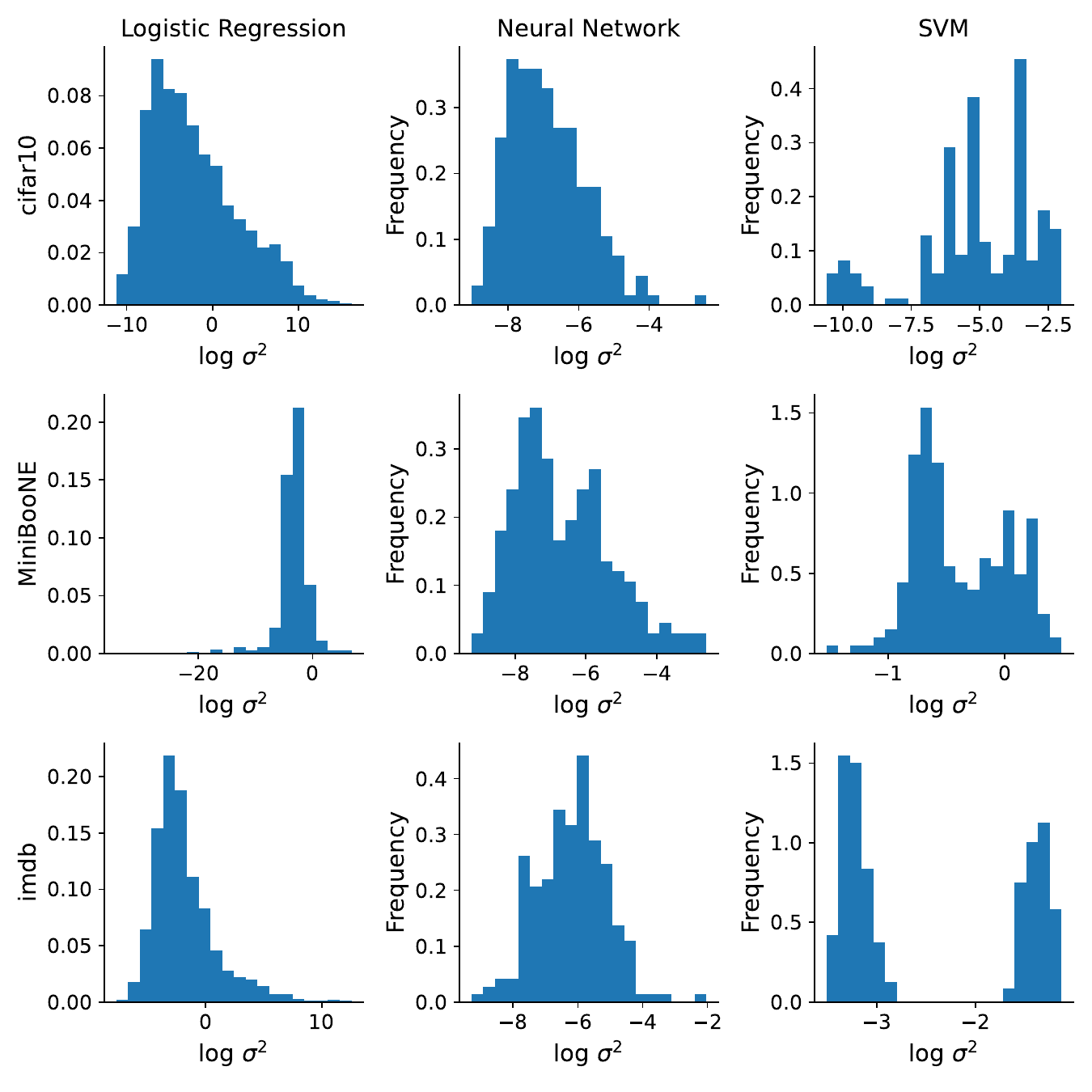}}
    \caption{\textbf{Histogram of estimated $\hat\beta(z)$ and $\log \hat \sigma^2(z)$}: the logistic regression model has a much higher $\beta$ compared with neural networks and SVM.}
    \label{fig: sigma and beta}
\end{figure}

Finally, we also show the relationship between distance to decision boundary and $\log|c(z)|$ in \Cref{fig: dist c}. Surprisingly, we observe a very different behavior for the logistic regression model on two different datasets. For the MiniBooNE dataset, the heterogeneity of $\log |c(z)|$ mainly comes from the points that are close to the decision boundary, and there is no significant trend when points are far from the decision boundary, as they all concentrate around 0. In contrast, we observe a clear positive correlation between the distance and the estimated $\log|c(z)|$ for the IMDB dataset. This difference implies that the scaling law parameters strongly depend on the data distribution, and they may have fundamentally different behavior when evaluating the same model for different datasets. 

\begin{figure}[ht]
    \centering
    \subfigure[MiniBooNE]{\includegraphics[width=0.4\linewidth]{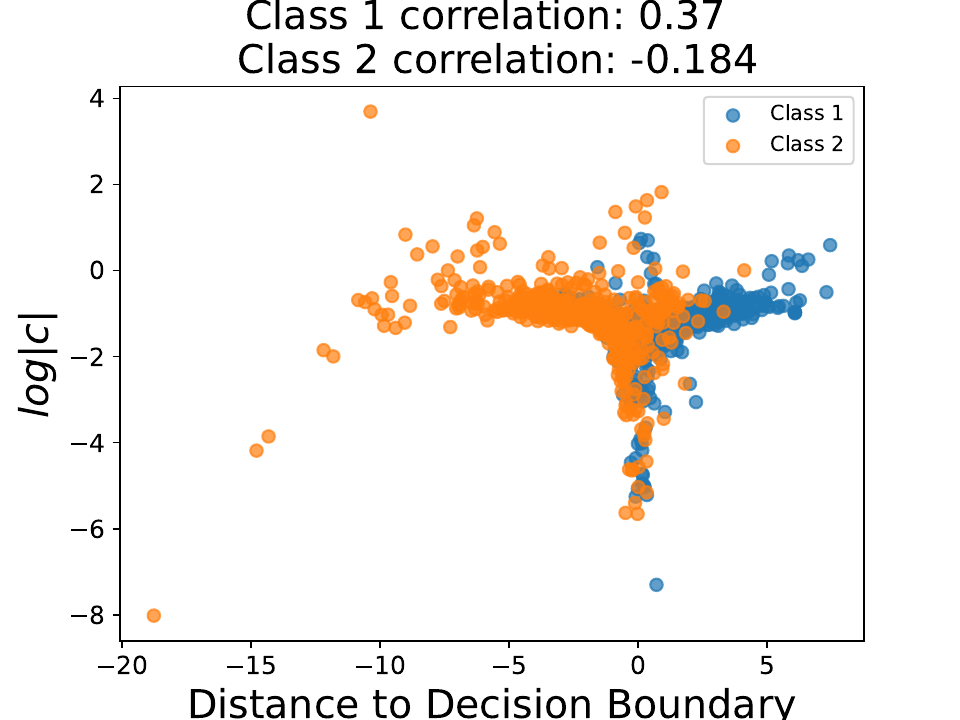}}
    \subfigure[IMDB]{\includegraphics[width=0.4\linewidth]{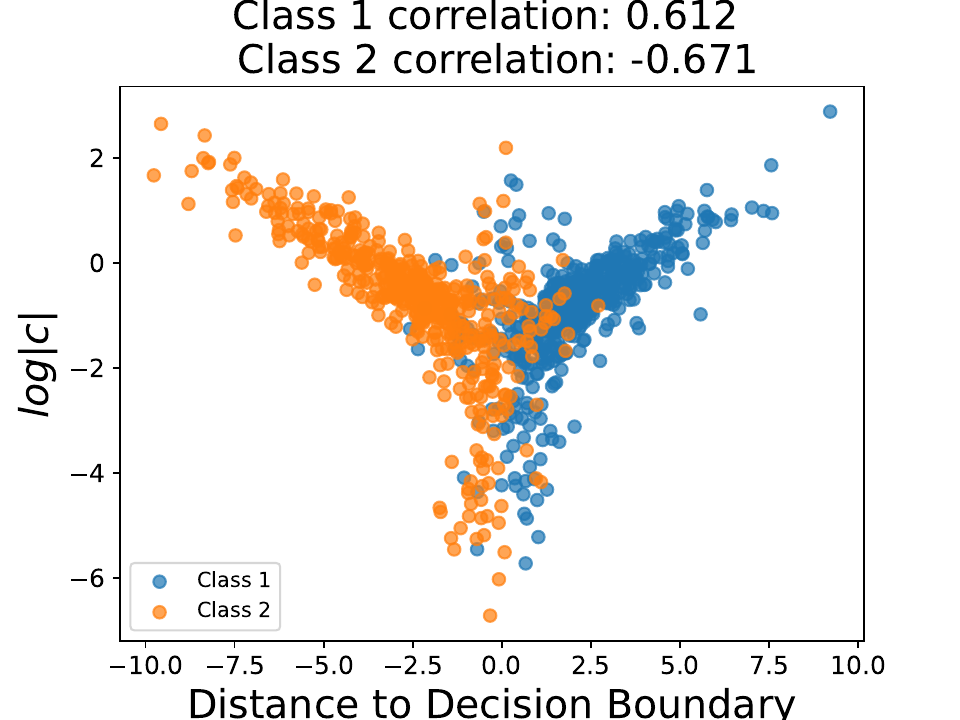}}
    \caption{\textbf{Distance to decision boundary vs. $\log|c(z)|$ for logistic regression model}. We fit a logistic regression model on all points being evaluated and then compute the distance of each data point to the decision boundary. The sign of distance is kept to distinguish points from each class.
    }
    \label{fig: dist c}
\end{figure}

\clearpage
\subsection{Accuracy of efficient scaling law estimators} \label{app:estimator-accuracy}

Here, we show three additional results related to our scaling law estimators. Like our results in \Cref{sec:estimation-accuracy}, all are conducted using the IMDB dataset with logistic regression. The first experiment provides intuition for our metrics by showing a scatterplot of the scaling law's predictions at each cardinality, and the true expectation $\psi_k(z)$ for each data point (estimated using 1000 samples). This result is shown in \Cref{fig:estimator-scatter}, where we observe very strong predictive accuracy for four dataset sizes. We report three measures of the prediction accuracy, where $\rho$ denotes Pearson correlation and $\tau$ denotes Spearman correlation. All three metrics are close to 1, particularly for the larger dataset sizes $k$.

\begin{figure}[h]
    \centering
    \includegraphics[width=\linewidth]{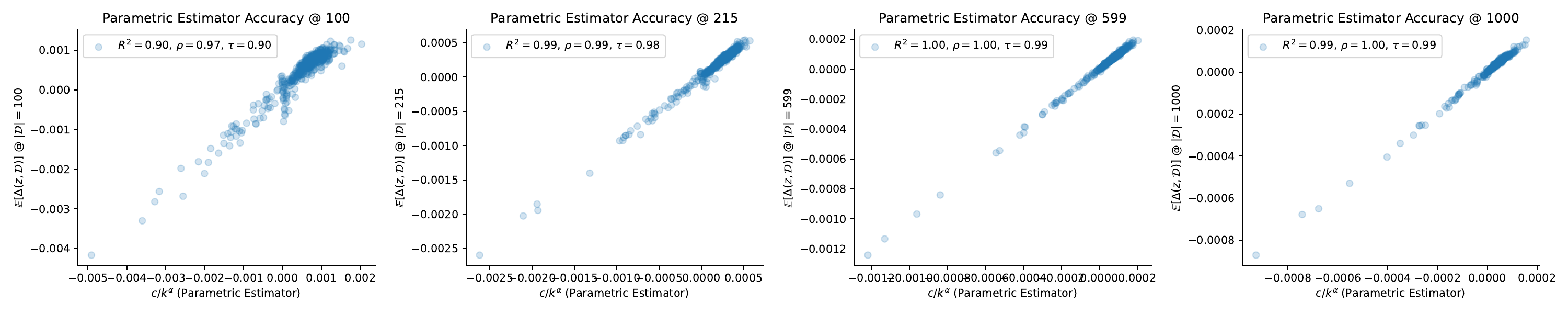}
    \caption{\textbf{Scaling law estimation accuracy scatterplots.} For four dataset sizes $k$, we plot the scaling law's predictions for the marginal contribution against the true expectations $\psi_k(z)$.}
    \label{fig:estimator-scatter}
\end{figure}

The next result is an expanded version of \Cref{fig:estimator-accuracy}, where we show the scaling law's accuracy at different $k$ values with different numbers of samples and according to three metrics: the $R^2$ score, Pearson correlation, and Spearman correlation. All three metrics are calculated using expectations $\psi_k(z)$ estimated with 1000 samples, similar to the previous result and our validation in \Cref{sec:validation}. Within the range where the scaling law is fit ($k \in [100, 1000]$), we generally observe high accuracy, particularly at larger cardinalities that have less noise, and when each estimator is fit with more samples. As shown in the main text, we see that the Pearson and Spearman correlation remain relatively high even as we extrapolate to dataset sizes an order of magnitude larger than where the scaling law is fit. However, the $R^2$ score quickly degrades once we exceed $k = 1000$, which we find is due to the predicted contributions not shrinking as quickly as the empirical ones.

\begin{figure}[h]
    \centering
    \includegraphics[width=\linewidth]{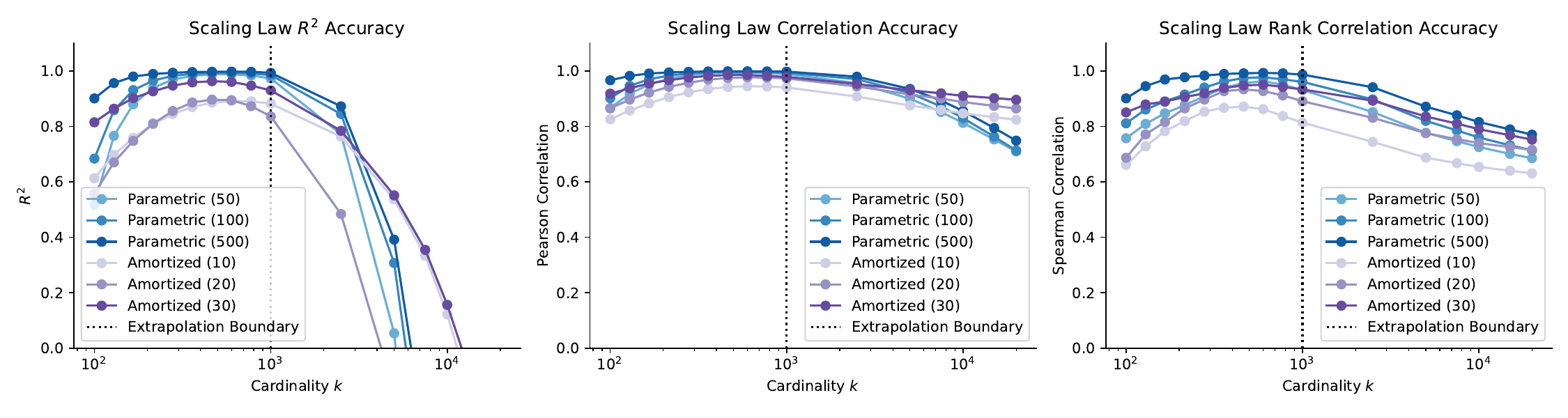}
    \caption{\textbf{Scaling law estimation accuracy at multiple cardinalities.} We evaluate the scaling law's predictions at cardinalities up to an order of magnitude larger than where it is fit, and according to three accuracy metrics: $R^2$, Pearson correlation and Spearman correlation.}
    \label{fig:estimator-accuracy-full}
\end{figure}

Finally, we show a more thorough analysis of the likelihood-based estimator's convergence characteristics. \Cref{fig:estimator-convergence} shows the accuracy at each cardinality via a colored line as a function of the number of samples used to fit the estimator. We also separately show results for interpolation ($k \in [100, 1000]$) and extrapolation ($k > 1000$). Overall, we see that for most interpolation cardinalities, the $R^2$ score converges to 1 within relatively few samples; the more difficult cases in the top row the smallest cardinalities. On the other hand, in the bottom row, we see that the performance has significant room for improvement even after fitting with 100 samples.

\begin{figure}[h]
    \centering
    \includegraphics[width=\linewidth]{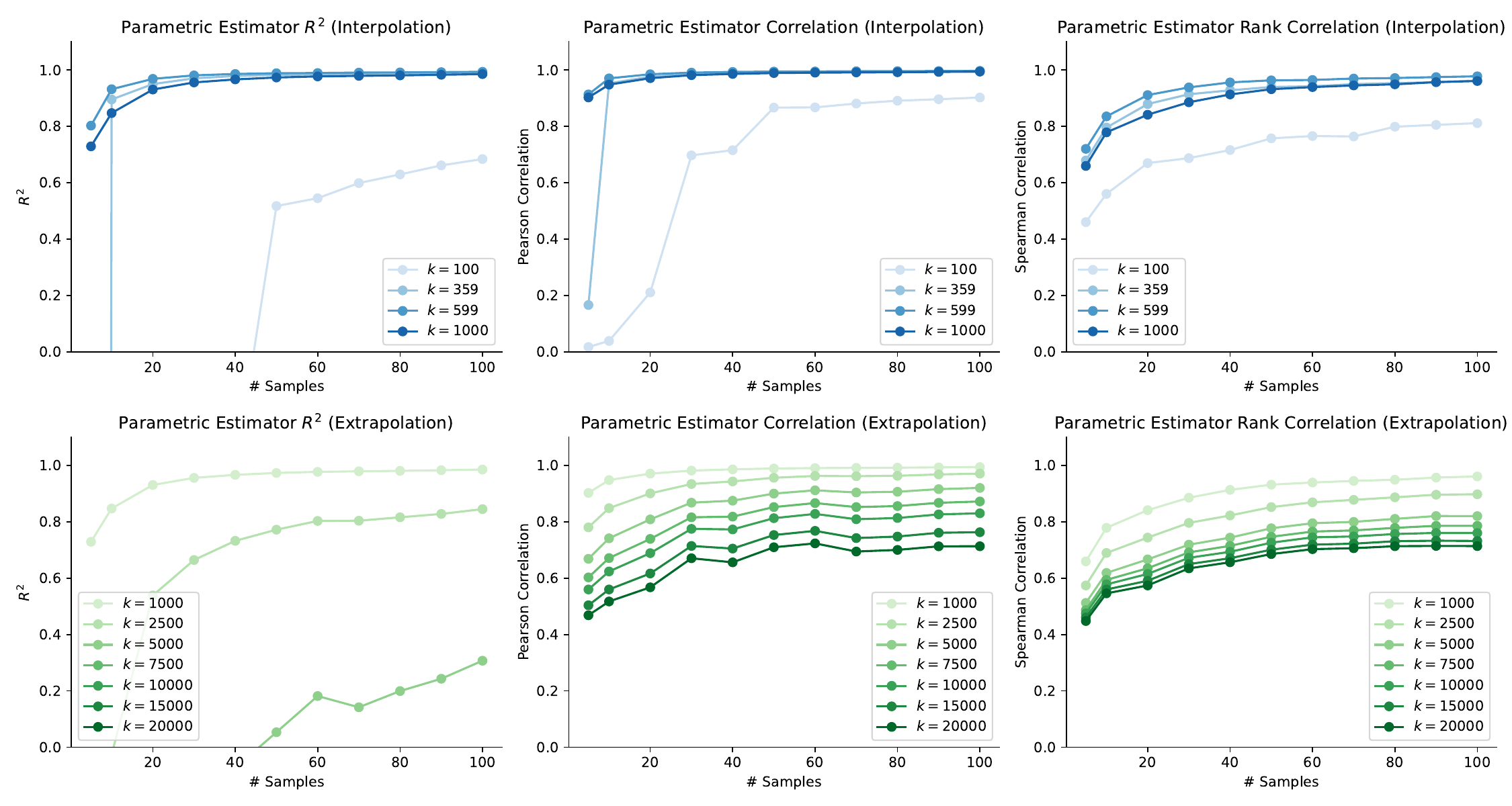}
    \caption{\textbf{Likelihood-based scaling law estimator convergence.} We verify that the likelihood-based estimator becomes more accurate as it is fit on more samples, where we separately show cardinalities in the fitting range (top) and beyond the fitting range (bottom).}
    \label{fig:estimator-convergence}
\end{figure}

\clearpage
\subsection{Application to data valuation} \label{app:valuation}

Here we provide additional results applying individualized scaling laws to data valuation. We first recall the definition of the distributional Shapley value \citep{ghorbani2020distributional}, which is one of many data valuation scores based on a data point's marginal contributions \citep{kwon2021beta, wang2022data, li2023robust}. Given a maximum cardinality $k_{\max}$, this method defines the valuation score $\psi(z)$ as a uniform average over the expected marginal contributions $\psi_k(z)$ up to $k = k_{\max}$. In practice, we generally also require a minimum cardinality $k_{\min}$ because models cannot be reliably trained without a reasonable amount of data (see the OpenDataVal package for example, \citealt{jiang2023opendataval}). With this in mind, we can define the valuation score as follows:

\begin{equation*}
    \psi(z) = \frac{1}{k_{\max} - k_{\min} + 1} \sum_{k = k_{\min}}^{k_{\max}} \psi_k(z).
\end{equation*}

Our experiments involve estimating these scores using a conventional Monte Carlo estimator, which averages marginal contributions $\Delta(z, \mcD)$ with dataset sizes sampled uniformly random, and comparing them to estimates derived from our scaling parameters. For a data point $z$ with parameters $c(z)$ and $\alpha(z)$, we can estimate the expected contribution at cardinality $k$ as $\hat{\psi}_k(z) = c(z) / k^{\alpha(z)}$, and we can average this across $k$ to estimate $\hat{\psi}(z)$. We test this approach when the scaling parameters are estimated using both our likelihood-based estimator from \Cref{sec: likelihood estimation} and the amortized estimator from \Cref{sec:amortization}. As the ground truth for the performance metrics, we use Monte Carlo estimates computed using $10000$ sampled marginal contributions, which is chosen to reliably approximate
the exact value.

The expanded versions of our main text results are shown in \Cref{fig:valuation-imdb} for IMDB, \Cref{fig:valuation-cifar} for CIFAR-10, and \Cref{fig:valuation-miniboone} for MiniBooNE. We also show results for an additional case, the adult census dataset in \Cref{fig:valuation-adult} \citep{dua2017uci}. All the results are generated by training logistic regression models. We observe that in all four cases, our scaling law-based estimates converge to the same result as the conventional Monte Carlo estimator, which attests to the scaling law accurately capturing the rate of decay of marginal contributions in practice. We observe that the parametric estimator becomes reliable when we use more than 20 samples, although it only outperforms the Monte Carlo estimator for MiniBooNE and the rank correlation metric for IMDB and adult census. On the other hand, the amortized estimator generally provides the best performance in the noisiest regime with just 10 samples, and it is the best estimator for all numbers of samples with the adult census dataset; however, in the case of CIFAR-10 it is surpassed by the non-amortized estimator for larger numbers of samples.

\begin{figure}[ht]
    \centering
    \includegraphics[width=\linewidth]{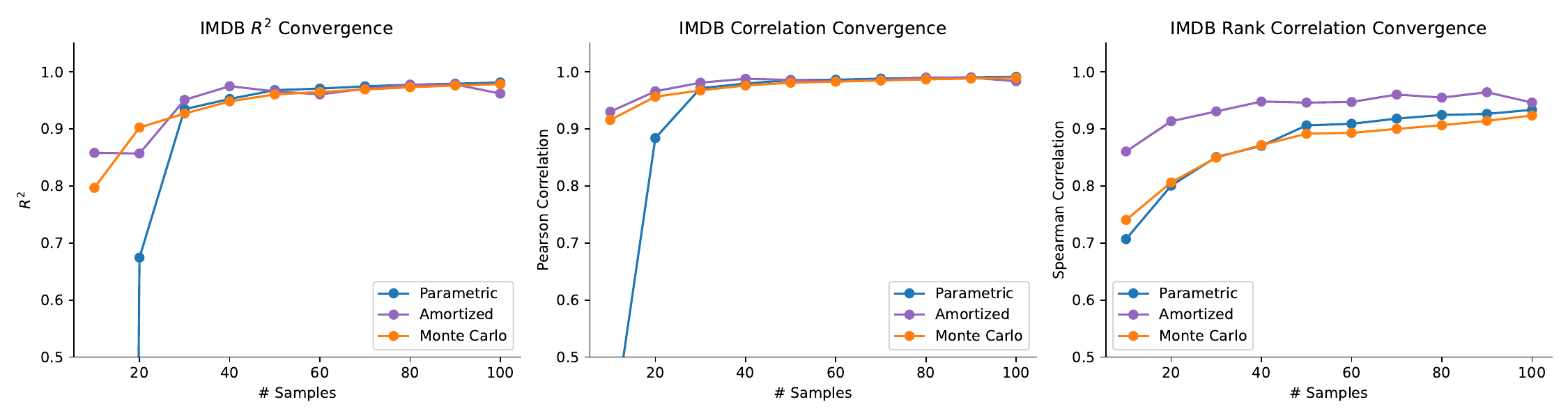}
    \caption{\textbf{Data valuation estimation accuracy for the IMDB dataset with logistic regression}.}
    \label{fig:valuation-imdb}
\end{figure}

\begin{figure}[ht]
    \centering
    \includegraphics[width=\linewidth]{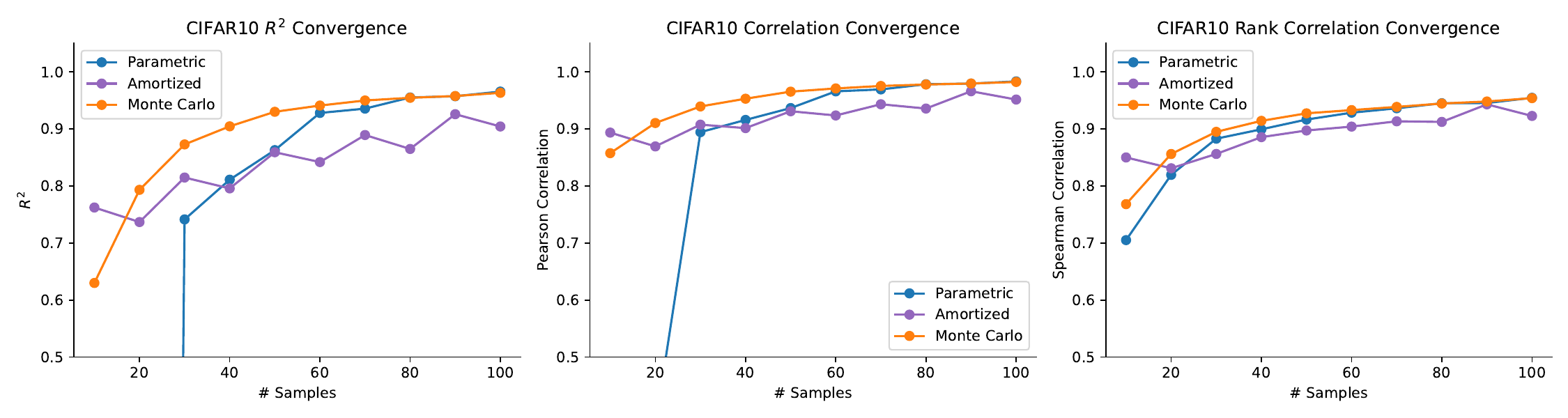}
    \caption{\textbf{Data valuation estimation accuracy for the CIFAR-10 dataset with logistic regression}.}
    \label{fig:valuation-cifar}
\end{figure}

\begin{figure}[ht]
    \centering
    \includegraphics[width=\linewidth]{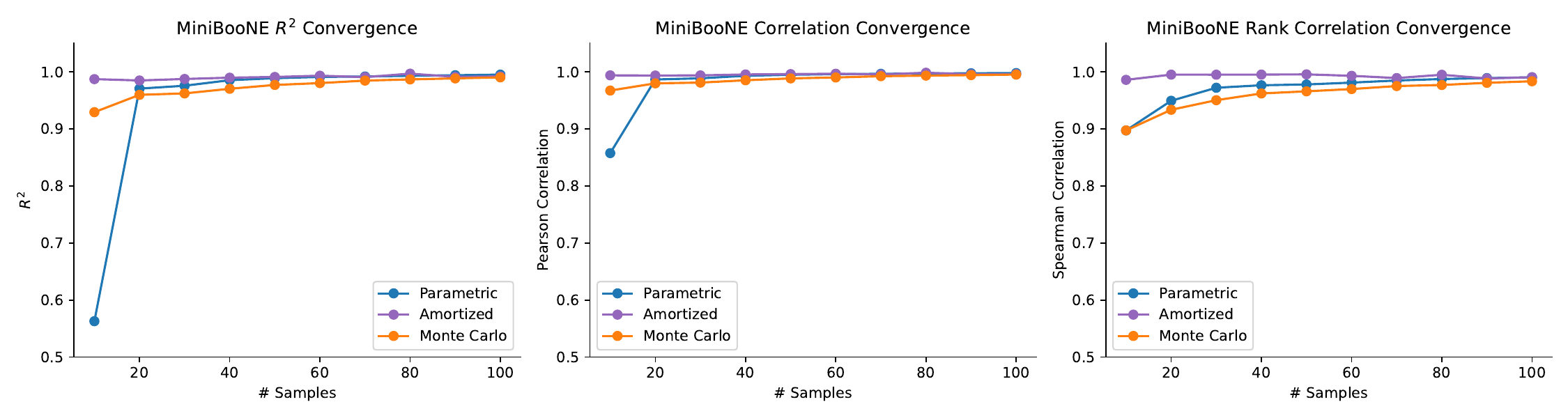}
    \caption{\textbf{Data valuation estimation accuracy for the MiniBooNE dataset with logistic regression}.}
    \label{fig:valuation-miniboone}
\end{figure}

\begin{figure}[ht]
    \centering
    \includegraphics[width=\linewidth]{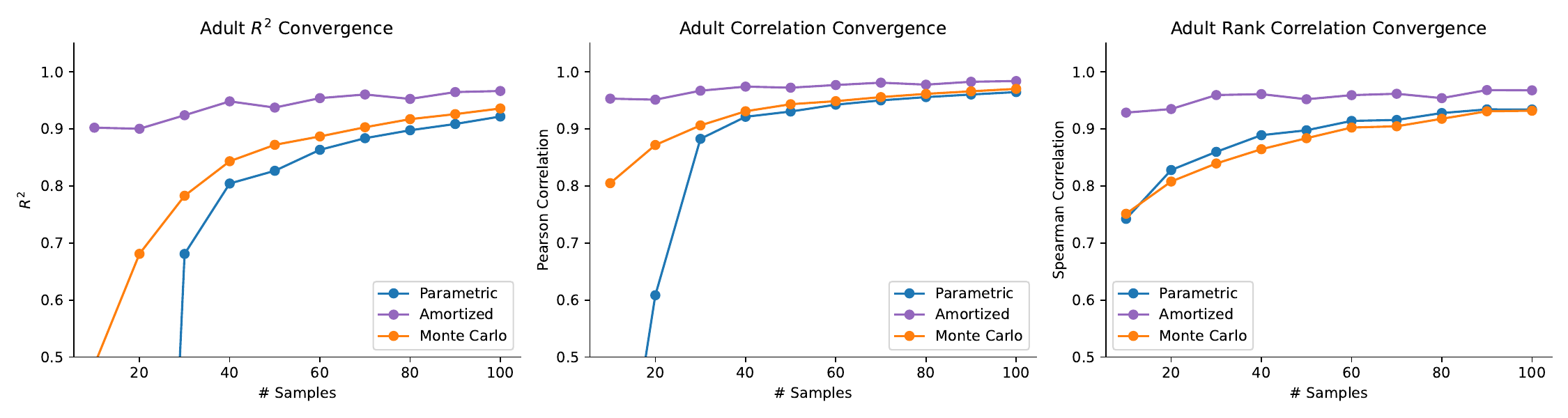}
    \caption{\textbf{Data valuation estimation accuracy for the adult census dataset with logistic regression}.}
    \label{fig:valuation-adult}
\end{figure}

\clearpage
\subsection{Application to point addition} \label{app:point-addition}

Here, we provide additional experimental results on the point addition task. In \Cref{sec: application subset}, we consider adding $20$ points to preceding datasets of size $100$ and $1000$, and here we consider the setting of adding $50$ points to preceding datasets of size $200$ and $2000$. Again, we will predict the expected contribution using the scaling law $\hat\psi_k(z) = c(z) / k^{\alpha(z)}$ with fitted parameters $c(z)$ and $\alpha(z)$ at $k=200$ and $k=2000$ to select the top points for \textit{Scaling 200} and \textit{Scaling 2000}. The results are shown in \Cref{tab: subset small 50} and \Cref{tab: subset large 50}, and the Scaling 200 and Scaling 2000 selections show superior performance on the corresponding smaller and larger datasets,
while performing badly in the other regime. Selections based on the Shapley value have worse performance than scaling law selections based on the dataset size. We notice that in \Cref{tab: subset small 50} there is one setting where random selection achieves the best performance, which may be due to interaction effects where the selected points are similar, so
the contribution from a whole set may not be as beneficial as expected.
In all other settings, we generally observe that random selection has the worst performance. 

\begin{table}[h]
    \caption{\textbf{Accuracy improvement (\%) with 50 points added to preceding datasets of size 200.} The settings are the same as \Cref{tab: subset small}, Scaling 200 achieves the best performance compared to other methods.}
    \label{tab: subset small 50}
    \vspace{0.2cm}
    \centering
    \begin{small}
    \begin{tabular}{cccc}
         \toprule
         Method& MiniBooNE & CIFAR-10 & IMDB \\
         \midrule
        Scaling 2000 & 0.27 $\pm$ 0.95      & 1.51 $\pm$ 0.58     & 0.71 $\pm$ 0.68     \\
        Scaling 200  & 0.33 $\pm$ 0.98      & \textbf{2.04 $\pm$ 0.59}     & \textbf{1.55 $\pm$ 0.61}     \\
        Shapley       & 0.28 $\pm$ 0.99      & 1.50 $\pm$ 0.58     & 1.19 $\pm$ 0.67     \\
        Random    & \textbf{0.44 $\pm$ 0.68}      & 1.16 $\pm$ 0.57     & 0.73 $\pm$ 0.53     \\
        \midrule
        Preceding  & 81.99 $\pm$ 0.97    & 80.51 $\pm$ 0.97    & 81.71 $\pm$ 0.94    \\
         \bottomrule
    \end{tabular}
    \end{small}
\end{table}

\begin{table}[h]
    \caption{\textbf{Accuracy improvement (\%) with 50 points added to preceding datasets of size 2000.} The settings are the same as \Cref{tab: subset small}, Scaling 2000 achieves the best performance compared to other methods.}
    \label{tab: subset large 50}
    \vspace{0.2cm}
    \centering
    \begin{small}
    \begin{tabular}{cccc}
         \toprule
         Method& MiniBooNE & CIFAR-10 & IMDB \\
         \midrule
        Scaling 2000  & \textbf{0.15 $\pm$ 0.33}    & \textbf{0.16 $\pm$ 0.13}    & \textbf{0.09 $\pm$ 0.09}    \\
        Scaling 200 & 0.14 $\pm$ 0.32    & 0.11 $\pm$ 0.10    & 0.06 $\pm$ 0.06    \\
        Shapley        & 0.14 $\pm$ 0.32    & 0.14 $\pm$ 0.13    & 0.08 $\pm$ 0.06    \\
        Random     & 0.01 $\pm$ 0.17    & 0.05 $\pm$ 0.13    & 0.03 $\pm$ 0.09    \\
        \midrule
        Preceding   & 84.51 $\pm$ 0.37   & 86.39 $\pm$ 0.35  & 86.38 $\pm$ 0.24  \\
         \bottomrule
    \end{tabular}
    \end{small}
\end{table}

\end{document}